\documentclass[lettersize,journal]{IEEEtran}
\usepackage{amsmath,amsfonts}
\def\BibTeX{{\rm B\kern-.05em{\sc i\kern-.025em b}\kern-.08em
		T\kern-.1667em\lower.7ex\hbox{E}\kern-.125emX}}
\usepackage{balance}

\usepackage[nocompress]{cite}
\usepackage{dsfont}
\usepackage{epsfig}
\usepackage{graphicx}
\usepackage{amsmath}
\usepackage{hyperref}
\usepackage{amssymb}
\usepackage{booktabs}
\usepackage{graphicx}
\usepackage{multirow}
\usepackage{amsthm}
\usepackage{soul}
\usepackage{pifont}
\usepackage{subfigure}
\usepackage{xcolor}
\usepackage[normalem]{ulem}
\usepackage{caption}
\usepackage{algorithm}
\usepackage{url}
\usepackage{graphicx}
\usepackage{tabularx}
\usepackage{tikz}
\usepackage{amssymb}
\usepackage{listings}
\usepackage{pythonhighlight}
\usetikzlibrary{calc, angles, quotes}
\usetikzlibrary{positioning, calc, matrix}
\newcommand{\methodname}{BootSC\space}
\newtheorem{theorem}{Theorem}
\newtheorem{lemma}{Lemma}

\begin{document}
	\title{Bootstrap Deep Spectral Clustering with\\Optimal Transport}
	\author{Wengang~Guo,~Wei~Ye,~Chunchun~Chen,~Xin~Sun,~\IEEEmembership{Member,~IEEE},\\Christian~B\"ohm,~Claudia~Plant,~and~Susanto~Rahardja,~\IEEEmembership{Fellow,~IEEE}
		\thanks{Wengang~Guo is with the College of Electronic and Information Engineering, Tongji University, Shanghai 201804, China (Email: guowg@tongji.edu.cn). }
			\thanks{Wei~Ye is with the College of Electronic and Information Engineering, Shanghai Institute of Intelligent Science and Technology, Tongji University, Shanghai 201804, China and Shanghai Innovation Institute, Shanghai 200231, China (Email: yew@tongji.edu.cn). (Corresponding author: Wei~Ye)}
			\thanks{Chunchun~Chen is with the Shanghai Research Institute for Intelligent Autonomous Systems, Tongji University, Shanghai 201210, China (Email: c2chen@tongji.edu.cn).}
			\thanks{Xin~Sun is with the Faculty of Data Science, City University of Macau, Taipa, Macau, China (Email: sunxin1984@ieee.org).}
			\thanks{Christian~B\"ohm and Claudia~Plant are with the Faculty of Computer Science, University of Vienna,  Vienna 1010, Austria (Email: \{christian.boehm,~claudia.plant\}@univie.ac.at).}
			\thanks{Susanto~Rahardja is with the Singapore Institute of Technology, Singapore 138683, Singapore (Email: susantorahardja@ieee.org).}}
		
	\markboth{Journal of \LaTeX\ Class Files,~Vol.~18, No.~9, September~2020}%
	{How to Use the IEEEtran \LaTeX \ Templates}

	\maketitle
	
	\begin{abstract}
		Spectral clustering is a leading clustering method.
		Two of its major shortcomings are the disjoint optimization process and the limited representation capacity. 
		To address these issues, we propose a deep spectral clustering model (named BootSC), which jointly learns all stages of spectral clustering---affinity matrix construction, spectral embedding, and $k$-means clustering---using a single network in an end-to-end manner. 
		BootSC leverages effective and efficient optimal-transport-derived supervision to bootstrap the affinity matrix and the cluster assignment matrix.
		Moreover, a semantically-consistent orthogonal re-parameterization technique is introduced to orthogonalize spectral embeddings, significantly enhancing the discrimination capability.
		Experimental results indicate that BootSC achieves state-of-the-art clustering performance. 
		For example, it accomplishes a notable 16\% NMI improvement over the runner-up method on the challenging ImageNet-Dogs dataset.
		Our code is available at \url{https://github.com/spdj2271/BootSC}.
	\end{abstract}

	
	\section{Introduction}
\IEEEPARstart{D}{eep} clustering models aim to detect underlying cluster structures within unlabelled data.
To train these models, creating effective and efficient supervision signals is necessary. 
Inadequate supervision could result in excessive computational costs \cite{cai2022efficient}, training instability \cite{zhan2020online}, and degenerate results \cite{caron2018deep}.

\begin{figure}[t]
	\newcommand{\AffinityGraphWithd}{0.304}
	\centering	
	\subfigure[Initial stage]{\includegraphics[height=\AffinityGraphWithd\columnwidth]{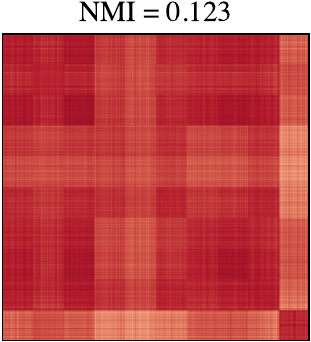}\label{img_initial}}
	\subfigure[Middle stage]{\includegraphics[height=\AffinityGraphWithd\columnwidth]{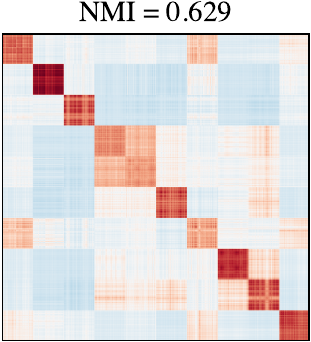}\label{img_middle}}
	\subfigure[Final stage~~~~~~~~~~~~~~~]{\includegraphics[height=\AffinityGraphWithd\columnwidth]{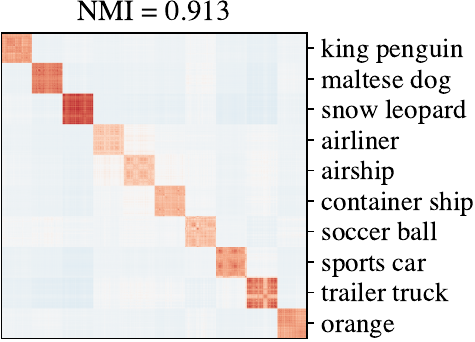}\label{img_final}}
	\caption{The affinity matrix learned by our BootSC at different training stages on the ImageNet-10 dataset \cite{chang2017deep}. 
		Training is batch-wise and from scratch.
		The affinity matrix gradually exhibits a diagonal block structure with clear cluster separation. 
	}
	\label{fig_2D}
\end{figure}

Classical deep clustering models \cite{xie2016unsupervised,ghasedi2017deep,guo2017improved,caron2019unsupervised,HsuL18,xu2024investigating} commonly adopt cluster assignments obtained by $k$-means on data representations as training supervision.
A major challenge with this $k$-means-style supervision is that data representations are assumed to follow simple isotropic Gaussian distributions.
This assumption frequently fails for high-dimensional real-world data \cite{yang2017towards}, which typically exhibit intricate nonconvex cluster structures.
Such structures cannot be accurately captured by $k$-means, leading to suboptimal clustering results.

In fact, certain sophisticated clustering methods like spectral clustering \cite{shi2000normalized} excel at detecting nonconvex cluster structures, potentially providing better supervision.
Motivated by this, deep spectral clustering models \cite{shaham2018spectralnet,yang2019deep,huang2019multi,huang2019multispectralnet,duan2019improving,affeldt2020spectral,ye2021spectral} have been proposed, which involve two separate stages---spectral embedding learning and $k$-means clustering. 
The spectral embedding learning is supervised by an affinity matrix, which is constructed within the embedding space of pre-trained networks like AutoEncoder \cite{bengio2006greedy} over the whole dataset.
Subsequently, the learned spectral embeddings are grouped by $k$-means.
While these models have achieved improved results, they can hardly scale to large datasets as constructing the whole affinity matrix leads to quadratic computational complexity.
Additionally, their clustering performance highly relies on the quality of pre-trained networks.
Furthermore, their disjoint learning process cannot provide optimal embeddings for clustering.

To handle the above issues, we propose a \textbf{Boot}strapped\footnote{In this paper, the term \textit{bootstrap} is used in its idiomatic sense rather than the statistical sense.} deep \textbf{S}pectral \textbf{C}lustering model (BootSC) using optimal transport.
Compared with existing models, our BootSC offers three key advantages: 
(1) It needs only mini-batch training \cite{lecun2015deep} and thereby enables scalability;
(2) It can learn a clustering-specific affinity matrix from scratch without requiring any pre-trained networks (see Figure \ref{fig_2D});
(3) It seamlessly integrates and jointly learns all stages of spectral clustering in an end-to-end manner.

\begin{figure}[t]
	\centering	
	\includegraphics[width=\columnwidth]{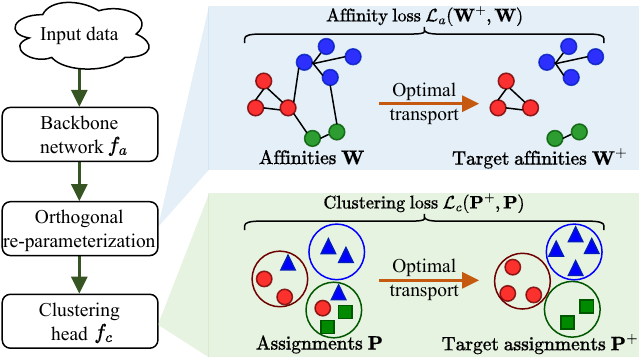}
	\caption{Illustration of our end-to-end BootSC, which simultaneously learns an affinity matrix and a cluster assignment matrix. We bootstrap the two matrices using effective and efficient optimal-transport-derived supervision for optimization.
	}
	\label{fig_framework}
\end{figure}

Our BootSC is end-to-end and learns in a bootstrapped manner (Figure \ref{fig_framework}). 
Given a mini-batch of samples, BootSC predicts an affinity matrix that captures pairwise similarities between samples and a cluster assignment matrix that indicates which cluster each sample belongs to. 
We leverage the spectral embedding and $k$-means clustering objectives as priors to bootstrap the two matrices.
Specifically, we compute two target matrices that respectively optimize the two objectives as supervision.
This optimization is a specific optimal transport problem, which can be efficiently solved with the off-the-shelf Sinkhorn’s fixed point iteration \cite{cuturi2013sinkhorn}.
Subsequently, we align the predicted matrices with the target matrices to update model parameters.
On the other hand, the orthogonality of spectral embeddings is crucial for spectral clustering.
Existing models commonly use a QR-decomposition-based layer \cite{shaham2018spectralnet} for orthogonalization.
However, this layer proves to fail in end-to-end joint training due to significant semantic inconsistency between the original and orthogonalized embeddings (Figure \ref{fig_orth} and Figure \ref{fig_inconsistency}).
To address this issue, we introduce an orthogonal re-parameterization technique that minimizes the semantic inconsistency through orthogonal Procrustes \cite{schonemann1966generalized}, thereby facilitating end-to-end joint training.

Our contributions are summarized as follows:
\begin{itemize}
	\item We present BootSC, a pioneering end-to-end deep spectral clustering model. 
	To our knowledge, BootSC is the first model that jointly learns all stages of spectral clustering without requiring any pre-trained networks.
	\item We uncover that both the objectives of spectral embedding and $k$-means clustering can be formulated as optimal transport problems, thereby facilitating effective and efficient joint optimization.
	\item We propose an orthogonal re-parameterization technique that enables semantically-consistent orthogonalization of network output, enhancing the discrimination power.
\end{itemize}

Our BootSC significantly outperforms state-of-the-art clustering methods on five benchmark image datasets. Notably, on the challenging ImageNet-Dogs dataset \cite{chang2017deep}, BootSC accomplishes a substantial 16\% performance improvement in terms of the NMI values over the most competitive baseline.

\section{Related Work}

\subsection{Deep Clustering} 
Deep clustering exploits the power of deep learning to facilitate clustering.
Classical works combine deep neural networks with traditional clustering methods such as $k$-means  \cite{caron2018deep,xie2016unsupervised,ghasedi2017deep,guo2017improved,caron2019unsupervised,HsuL18,xu2024investigating,guo2021deep},  subspace clustering \cite{cai2022efficient,ji2017deep,WangCGZJ21,SunWZLGZZ22}, agglomerative clustering \cite{yang2016joint,LiLWZH20}, and spectral clustering \cite{shaham2018spectralnet,yang2019deep,huang2019multi,sun2022network,huang2019multispectralnet,duan2019improving,affeldt2020spectral,guo2024deep,ye2021spectral}.
Recently, some works focus on optimizing clustering by neighborhood consistency \cite{dang2021nearest,wu2019deep,van2020scan} and mutual information maximization \cite{ji2019invariant,WangCFZ23,yang2022learning}.
Inspired by contrastive learning \cite{chen2020simple,he2020momentum}, contrastive clustering \cite{dang2021nearest,li2021contrastive,shen2021you,PengLLQCL24,WuZRHPHHH24} have exhibited promising clustering results.
In very recent, some literature \cite{kwon2023image,zhang2023clusterllm,viswanathan2023large}  has started to employ large language models to improve clustering.
Kwon \textit{et al.} \cite{kwon2023image} introduce a novel image clustering paradigm that produces diverse clustering results for a given image dataset according to the textual criteria provided by users.
Zhang \textit{et al.}\cite{zhang2023clusterllm}  query ChatGPT \cite{achiam2023gpt} with prompts such as ``do sample $\mathbf{x}_i$ and $\mathbf{x}_j$ belong to the same cluster'', and then refine embeddings based on the ChatGPT answers.
These methods commonly adopt the Euclidean distance-based measure for cluster detection, whereas Euclidean distance can be invalid for highly complex data.
In contrast, our \methodname is a connectivity-based clustering method that effectively groups data according to data similarities.

\subsection{Traditional Spectral Clustering} 
Spectral clustering formulates the clustering task as a graph Min-Cut problem.  
However, directly solving the Min-Cut problem often yields degenerate clustering results where a single vertex forms a cluster.
To address this challenge, Shi \textit{et al.} \cite{shi2000normalized} and Hagen \textit{et al.}\cite{hagen1992new} respectively incorporate different penalties into the Min-Cut problem for promoting balanced cluster sizes.
For a detailed explanation of these modifications, please refer to the comprehensive tutorial \cite{von2007tutorial}.
These modifications have popularized spectral clustering and fostered numerous further improvements.

Some works \cite{li2024deep,zhao2023spectral,li2018dynamic,huang2012affinity,ren2020consensus,TangZLLWZW19} focus on constructing better affinity matrices to improve spectral clustering.
Zelnik \textit{et al.} \cite{zelnik2004self} propose automatically determining the temperature parameter of the Gaussian kernel according to the local neighborhood statistics of each sample.
Zhu \textit{et al.} \cite{zhu2014constructing} construct robust affinity matrices using the hierarchical structure of random forests \cite{breiman2001random}.
On the other hand, several works aim to reduce the computational cost of spectral embedding via Nystr\"om approximation \cite{fowlkes2004spectral}, power iteration \cite{lin2010power,ye2016fuse}, landmarks \cite{cai2014large}, label propagation \cite{wang2020large}, and stochastic gradient optimization \cite{han2017mini}.
Additionally, some works replace $k$-means clustering with alternative methods to detect clusters, such as spectral rotation \cite{shi2003multiclass,huang2013spectral,chen2017scalable} and pivoted QR-decomposition \cite{damle2019simple}.
Another line of works \cite{wang2021fast,huang2013spectral,yang2016unified,kang2018unified,WanZSYYY24,damle2019simple,wang2020large,zhong2021improved} explore the joint optimization of spectral embedding and clustering.
For example, Huang \textit{et al.} \cite{huang2013spectral} bridge the two steps using a rotation matrix, aiming to minimize the reconstruction error between the rotated spectral embedding matrix and the cluster assignment matrix. 
SE-ISR \cite{wang2021fast} further proposes a parameter-free method to trade off the Min-Cut cost and the reconstruction error.
A problem of these shallow methods is that they suffer from high-dimensional complex data due to their inferior representation capability.

\subsection{Deep Spectral Clustering} 
Deep spectral clustering integrates traditional spectral clustering with deep neural networks for effective clustering.
Tian \textit{et al.} \cite{tian2014learning} observe the similarity between the optimization objectives of autoencoders and spectral clustering.
They thus propose replacing spectral embedding with a deep autoencoder that is trained to reconstruct the pre-defined affinity matrix. 
However, obtaining such an affinity matrix can be challenging for complex data, and the matrix size can become prohibitively large for large datasets.
SpectralNet \cite{shaham2018spectralnet} proposes to directly embed raw data into the eigenspace of a given affinity matrix. 
SpectralNet has inspired numerous extensions.
Yang \textit{et al.} \cite{yang2019deep} enhance the robustness of SpectralNet's embeddings using a dual autoencoder.
\cite{huang2019multi,huang2019multispectralnet} respectively extend SpectralNet to handle multi-view data.
Another line of works \cite{ye2021spectral, duan2019improving, affeldt2020spectral} highlight the benefits of joint optimization and propose joint spectral embedding learning and clustering.
However, these methods only focus on a subset of stages within the spectral clustering pipeline, while our BootSC unifies and jointly optimizes the entire pipeline.

\subsection{Sinkhorn's Fixed Point Iteration}
We employ Sinkhorn's fixed point iteration for affinity learning and $k$-means clustering. 
This algorithm is rooted in the theory of entropy regularized optimal transport \cite{cuturi2013sinkhorn}, which iteratively rescales a matrix to achieve a doubly stochastic form.
Owing to its efficiency and differentiability, this algorithm has been widely integrated into deep learning frameworks such as
domain adaptation \cite{flamary2016optimal}, unsupervised representation learning \cite{asano2019self}, contrastive learning \cite{caron2020unsupervised}, and object detection \cite{ge2021ota}.
In contrast, we uniquely focus on extending it to spectral clustering.

	\section{Preliminaries}
\subsection{Spectral Clustering}
Let's consider the problem of grouping $N$ samples $\mathbf{X}=[\mathbf{x}_1,\cdots,\mathbf{x}_N]$ into $K$ disjoint clusters.
Traditional shallow spectral clustering involves three stages.
First, the affinity matrix $\mathbf{W}=\mathbf{D}^{-1}\mathbf{S}\in \mathbb{R}^{N\times N}$ is constructed, where element $S_{ij}=\exp(-{\|\mathbf{x}_i-\mathbf{x}_j\|^2}/(2{\sigma^2}))$ measures the similarity between sample $\mathbf{x}_i$ and $\mathbf{x}_j$, $\sigma$ is the Gaussian kernel bandwidth, and $\mathbf{D}$ is a diagonal matrix with $D_{ii}=\sum_{j=1}^N S_{ij}$.
Subsequently, the spectral embedding objective \cite{meilua2001random} is maximized to compute spectral embeddings $\mathbf{Z}\in \mathbb{R}^{N\times K}$:
\begin{equation}
	\underset{\mathbf{Z}\in \mathbb{R}^{N\times K}}{\max}\operatorname{Tr}(\mathbf{Z}^\intercal \mathbf{W}\mathbf{Z})~ \text{s.t.}~ \mathbf{Z}^\intercal\mathbf{Z}=\mathbf{I}_K,
	\label{eq_ytly}
\end{equation}
where $\operatorname{Tr}(\cdot)$ denotes the trace function and $\mathbf{I}_K$ is the $K$-dimensional identity matrix.
The solution for $\mathbf{Z}$ comprises the $K$ largest eigenvectors of $\mathbf{W}$, capturing the underlying cluster structures.
Finally, $k$-means clustering is applied to the rows of $\mathbf{Z}$ for final cluster assignments.

\subsection{Optimal Transport}
Supposing there are $N_s$ suppliers holding $\mathbf{s}=[s_1,\cdots,s_{N_s}]$ units of goods and $N_d$ demanders needing $\mathbf{d}=[d_1,\cdots,d_{N_d}]$ units of goods,	and the transportation cost for each unit of good from the $i$-th supplier to the $j$-th demander is denoted by $C_{ij}$, optimal transport aims to find an optimal transportation plan $\mathbf{Q}$ that minimizes the total transportation cost:
\begin{equation}
	\min_{\mathbf{Q}\in \mathbb{R}^{N_s\times N_d}_+}\sum_{i=1}^{N_s}\sum_{j=1}^{N_d} {Q}_{ij}C_{ij}~\text{s.t.}~\mathbf{Q}\mathbf{1}_{N_d}=\mathbf{s},~{\mathbf{Q}}^\intercal\mathbf{1}_{N_s}=\mathbf{d},
	\label{eq_ot}
\end{equation}
where $\mathbf{1}_{N_d}$ is the $N_d$-dimensional all-ones vector.
The two constraints ensure that all goods from the suppliers are transported to the demanders.
This problem can be efficiently solved using Sinkhorn's fixed point iteration \cite{cuturi2013sinkhorn}.

\section{Method: BootSC}
\subsection{Bootstrapping Affinities}
\label{sec_affinities}
The shortcomings of traditional spectral clustering are that its separate stages hamper joint optimization and its shallow nature limits the representation capability.
Our BootSC aims to resolve these challenges (Figure \ref{fig_framework}). 
Specifically, we utilize a deep neural network $f_{{a}}$ to map each sample $\mathbf{x}$ into a $D$-dimensional space $\mathbf{z}=f_{{a}}(\mathbf{x})$.
Given  spectral embeddings $\mathbf{Z}=[\mathbf{z}_1,\cdots,\mathbf{z}_N]\in \mathbb{R}^{N\times D}$, we orthogonalize $\mathbf{Z}$ column-wise using a re-parameterization technique (described in Section \ref{sec_ort}) and $\ell_2$-normalize each row of $\mathbf{Z}$.
We model the affinity matrix $\mathbf{W}$ as a doubly stochastic matrix:
\begin{equation}
	W_{ij}=\frac{\exp(\mathbf{z}_i\mathbf{z}_j^{\intercal} /\tau)}{\sum_{j =1 }^N \exp(\mathbf{z}_{i} \mathbf{z}_j^{\intercal}/\tau)},
	\label{eq_w}
\end{equation}
where $W_{ij}$ indicates the affinity between sample $\mathbf{x}_i$ and $\mathbf{x}_j$, $\mathbf{z}_i \mathbf{z}_j^{\intercal}$ measures the cosine similarity between $\mathbf{z}_i$ and $\mathbf{z}_j$ as each $\mathbf{z}$ is $\ell_2$-normalized, and $\tau$ is the temperature parameter of the softmax function that is important for clustering multi-scale data \cite{zelnik2004self}.
Following \cite{radford2021learning}, we learn $\tau$ as a trainable parameter to avoid tedious hyperparameter tuning.

We aim to maximize the spectral embedding objective in Equation (\ref{eq_ytly}) to learn affinities $\mathbf{W}$ and spectral embeddings $\mathbf{Z}$.
Equation (\ref{eq_ytly}) can be written as:
\begin{equation}
	\operatorname{Tr}(\mathbf{Z}^\intercal \mathbf{W}\mathbf{Z})=\operatorname{Tr}( \mathbf{W}\mathbf{Z}\mathbf{Z}^\intercal)\stackrel{\text{def}}{=}\operatorname{Tr}( \mathbf{W}\mathbf{G})=\sum_{i=1}^N\sum_{j=1}^N W_{ij}G_{ij},
\end{equation}
where the first equality uses the cyclic property of the trace function, the second equality defines the Gram matrix $\mathbf{G}$ as $\mathbf{Z}\mathbf{Z}^\intercal$, and the final equality uses the symmetry of $\mathbf{G}$.
Hence, our learning objective can be rewritten as:
\begin{equation}
	\min-\sum_{i=1}^N\sum_{j=1}^N W_{ij}G_{ij}.
	\label{eq_wzzt}
\end{equation}

Unlike traditional spectral clustering that only optimizes $\mathbf{Z}$ under a fixed $\mathbf{W}$, BootSC jointly learns both $\mathbf{W}$ and $\mathbf{Z}$ through a bootstrap procedure that alternately performs: (1) the target estimation step to refine $\mathbf{W}$, and (2) the parameter update step to refine $\mathbf{Z}$. 

\textbf{In the target estimation step,} we freeze $\mathbf{G}$ (i.e., freeze $\mathbf{Z}$) and find a target $\mathbf{W}$ that minimizes the spectral embedding cost in Equation \eqref{eq_wzzt} as supervision.
Equation \eqref{eq_wzzt} suffers from the trivial solution $\mathbf{W}=\mathbf{I}_N$ that only captures self-affinities and is ineffective for training (Figure \ref{fig_trivialsolution}).
To address this issue, we compel the model to mine meaningful cross-affinities in the off-diagonal elements by simply removing the diagonal elements of $\mathbf{W}$ and $\mathbf{G}$ \footnote{For simplicity, we continue to use the notations $\mathbf{W}$ and $\mathbf{G}$ to represent their off-diagonal versions.}: 
\begin{equation}
	\min_{\mathbf{W}}-\sum_{i=1}^N\sum_{j=1}^{N-1} {W}_{ij}G_{ij} ~\text{s.t.}~\mathbf{W}\mathbf{1}_{N-1}=\mathbf{1}_N,~{\mathbf{W}}^\intercal\mathbf{1}_N=\mathbf{1}_{N-1},
	\label{wq_wplus}
\end{equation}
where the two constraints ensure the resulting solution remains doubly stochastic, consistent with the modeled affinity matrix. 


Next, we tackle Equation \eqref{wq_wplus}, which seems a challenging linear programming problem.
Fortunately, it can be treated as an optimal transport problem, enabling efficient resolution with well-established solvers.
Concretely, we set $-\mathbf{G}$  as the transportation cost for a hypothetical transport task and $\mathbf{W}\in\mathcal{W}$ as the corresponding transportation plan, where $\mathcal{W}= \{\mathbf{W}\in \mathbb{R}^{N\times (N-1)}_+\mid \mathbf{W}\mathbf{1}_{N-1}=\mathbf{1}_N,~{\mathbf{W}}^\intercal\mathbf{1}_N=\mathbf{1}_{N-1}\}$ is the transportation polytope.
Hence, finding the target affinity matrix is equivalent to solving for the optimal transportation plan that minimizes the total transportation cost.
We tackle this problem using Sinkhorn's fixed point iteration \cite{cuturi2013sinkhorn}, which adds an entropic regularization for efficient resolution:
\begin{equation}
	\min_{\mathbf{W} \in \mathcal{W}}\underbrace{-\sum_{i=1}^N\sum_{j=1}^{N-1} {W}_{ij}G_{ij}}_{\text{transportation cost}} -\eta\underbrace{\sum_{i=1}^{N}\sum_{j=1}^{N-1}W_{ij}\log W_{ij}}_{\text{entropic regularization}}, 
	\label{wq_wplus2}
\end{equation}
where $\eta$ is the trade-off parameter. We evaluate the effect of different $\eta$ values (including $\eta=0$) on clustering performance in Section \ref{sec_sk}.
The solution to Equation \eqref{wq_wplus2} denoted by $\mathbf{W}^+$ is given by:
\begin{equation}
	{\mathbf{W}^+}=\operatorname{Diag}(\boldsymbol{\alpha})  \exp({\mathbf{G}}/{\eta }) \operatorname{Diag}(\boldsymbol{\beta}), 
	\label{eq_wplus}
\end{equation}
where $\operatorname{Diag}(\cdot)$ denotes a diagonal matrix constructed from the given vector and $\exp(\cdot)$ denotes the element-wise exponential operation.
$\boldsymbol{\alpha}\in \mathbb{R}^N$ and $\boldsymbol{\beta}\in \mathbb{R}^{N-1}$ are re-normalization vectors obtained by alternately performing the following updates:
\begin{equation}
	\boldsymbol{\alpha}  \leftarrow 1/ \bigl( \exp({\mathbf{G}}/{\eta })\boldsymbol{\beta} \bigr);  ~ \boldsymbol{\beta}  \leftarrow 1/\bigl( \exp({\mathbf{G}^\intercal }/{\eta })\boldsymbol{\alpha}\bigr).
\end{equation}

\textbf{In the parameter update step,}  we refine spectral embeddings $\mathbf{Z}$ by training the model to minimize the cross-entropy loss between the modeled affinities $\mathbf{W}$ and the target affinities $\mathbf{W}^+$ (termed affinity loss): 
\begin{equation}
	\begin{aligned}
		&\mathcal{L}_{{a}}(\mathbf{W}^+,\mathbf{W})=-\sum_{i=1}^N\sum_{j=1}^{N-1}W^+_{ij}\log W_{ij}\\
		&=\underbrace{-\sum_{i=1}^N\sum_{j=1}^{N-1}{W}^+_{ij}\mathbf{z}_j^\intercal \mathbf{z}_i/\tau}_{\text{spectral embedding cost}} +\underbrace{\sum_{i=1}^{N}\log \sum_{j=1}^{N-1}\exp(\mathbf{z}_{j}^\intercal \mathbf{z}_i/\tau)}_{\text{balanced affinity distribution}},
	\end{aligned}
	\label{loss_spectral}
\end{equation}
where the first term equivalently maximizes the scaled spectral embedding objective $\operatorname{Tr}(\mathbf{Z}^\intercal \mathbf{W}^+\mathbf{Z})/\tau$ and the second term uniformly separates all spectral embeddings over the spherical embedding space, promoting a balanced affinity distribution. 
The partial derivative of $\mathcal{L}_{{a}}$ with respect to $\mathbf{Z}$ is:
\begin{equation}
	\frac{\partial\mathcal{L}_{{a}}}{\partial \mathbf{Z}}=\frac{2}{\tau}({\mathbf{W}}-{\mathbf{W}^+})\mathbf{Z}=\mathbf{0}.
\end{equation}  
This derivative indicates that as training progresses, the spectral embeddings converge to the eigenvectors of the nullspace of the residual affinity matrix $\mathbf{W}-{\mathbf{W}^+}$.

\subsection{Bootstrapping Cluster Assignments} 
To achieve end-to-end spectral clustering, we extend the above bootstrap procedure to $k$-means clustering.
Specifically, we append a clustering head $f_{{c}}$ parametrized by $\ell_2$-normalized prototypes $\{\boldsymbol{\mu}_k\}_{k=1}^K$ onto the backbone network.
The probabilistic cluster assignment of sample $\mathbf{x}_i$ to the $k$-th cluster is predicted by:
\begin{equation}
	P_{ik}=\frac{\exp(\boldsymbol{\mu}_k^\intercal \mathbf{z}_i/\tau)}{\sum_{k=1}^K \exp(\boldsymbol{\mu}_k^\intercal \mathbf{z}_i/\tau)}.
	\label{eq_assignment}
\end{equation}

\begin{figure*}[t]
	\centering
	\newcommand{\scale}{0.42}
	\subfigure[Initial embeddings]{
		\begin{tikzpicture}[scale=\scale]
			\draw[-latex] (-5.3,0) -- (5.3,0) node[below] {};
			\draw[-latex] (0,-1.2) -- (0,5.3) node[left]{};
			\draw[thick,-latex, red,dashed] (0,0) -- ($(30:5)$) coordinate (A);
			\draw[thick, -latex, red] (0,0) -- ($(160:5)$)  coordinate (B);
			\node at ( 0.8660254  *5-1, 0.5 *5+0.5)[font=\scriptsize, black] {$\mathbf{z}_2$=[0.87, 0.5]};
			\node at ( -0.93969262*5+1.6,  0.34202014*5+0.5)[font=\scriptsize, black] {$\mathbf{z}_1$=[-0.94, 0.34]};
		\end{tikzpicture}
	}
	\label{fig_orth_a}
	\subfigure[QR-decomposition (SpectralNet \cite{shaham2018spectralnet})]{
		\begin{tikzpicture}[scale=\scale]
			\draw[-latex] (-5.3,0) -- (5.3,0) node[below] {};
			\draw[-latex] (0,-1.2) -- (0,5.3) node[left]{};
			\draw[thick,-latex, opacity=0.2, red,dashed] (0,0) -- ($(30:5)$) coordinate (A);
			\draw[thick, -latex, opacity=0.2, red] (0,0) -- ($(160:5)$)  coordinate (B);
			\draw[thick,-latex, blue,dashed] (0,0) -- (-0.67769547*5,0.73534267*5) coordinate (C);
			\draw[thick, -latex, blue] (0,0) -- (0.73534267*5,  0.67769547*5)  coordinate (D);
			\node at (-0.67769547*5+1,0.73534267*5+0.5)[font=\scriptsize, black] {$\mathbf{z}_2$=[-0.68, 0.74]};
			\node at (0.73534267*5-1,  0.67769547*5+0.5)[font=\scriptsize, black] {$\mathbf{z}_1$=[0.74, 0.68]};
			\coordinate (E) at ($(C)!0.5!(D)$);
			\coordinate (F) at ($(0,0)!0.2!(E)$);
			\draw[] (F) -- ($(0,0)!(F)!(C)$);
			\draw[] (F) -- ($(0,0)!(F)!(D)$);
			\draw[thick,->] (1,-0.5) -- (0.2,0.1);
			\node at (0.9,-0.6)[font=\scriptsize,right,black] {\textbf{Orthogonality}};
		\end{tikzpicture}
	}
	\label{fig_orth_b}
	\subfigure[Orthogonal Procrustes (Ours)]{
		\begin{tikzpicture}[scale=\scale]
			\draw[-latex] (-5.3,0) -- (5.3,0) node[below] {};
			\draw[-latex] (0,-1.2) -- (0,5.3) node[left]{};
			\draw[thick,-latex, opacity=0.2, red,dashed] (0,0) -- ($(30:5)$) coordinate (A);
			\draw[thick, -latex, opacity=0.2, red] (0,0) -- ($(160:5)$)  coordinate (B);
			\draw[thick,-latex, blue,dashed] (0,0) --(0.64278761*5,  0.76604444*5)  coordinate (C);
			\draw[thick, -latex, blue] (0,0) -- (-0.76604444*5,  0.64278761*5)   coordinate (D);
			\node at (0.84278761*5-1,  0.76604444*5+0.5)[font=\scriptsize, black] {$\mathbf{z}_2$=[0.64, 0.77]};
			\node at (-0.86604444*5+1.5,  0.64278761*5+0.5)[font=\scriptsize, black] {$\mathbf{z}_1$=[-0.77, 0.64]};
			\draw[|<->|] ($(A)$) -- node[midway, above,sloped,font=\scriptsize, black] {$\min$} ($(C)$); 
			\draw[|<->|] ($(B)$) -- node[midway, above,sloped,font=\scriptsize, black] {$\min$} ($(D)$); 
			\coordinate (E) at ($(C)!0.5!(D)$);
			\coordinate (F) at ($(0,0)!0.2!(E)$);
			\draw[] (F) -- ($(0,0)!(F)!(C)$);
			\draw[] (F) -- ($(0,0)!(F)!(D)$);
			\draw[thick,->] (1,-0.5) -- (0.2,0.1);
			\node at (0.9,-0.6)[font=\scriptsize,right,black] {\textbf{Orthogonality}};
			\draw[thick,->] (-4.1,1.) -- (-4.1,2.4);
			\node at (-1.6,0.65)[font=\scriptsize,left,black] {\textbf{Inconsistency}};
		\end{tikzpicture}
	}
	\caption{Geometric interpretation of different orthogonalization methods. 
		(a) Initial 2D embeddings $\mathbf{Z}=[\mathbf{z}_1,\mathbf{z}_2]$ of two samples that are not orthogonal.
		(b) Orthogonalized embeddings $\mathbf{Z}_\text{new}$ generated by \cite{shaham2018spectralnet} exhibit large semantic inconsistency $\|\mathbf{Z}-\mathbf{Z}_\text{new}\|_F=2.31$.
		(c) Our orthogonalized embeddings $\mathbf{Z}_\text{new}$ have minimal semantic inconsistency $\|\mathbf{Z}-\mathbf{Z}_\text{new}\|_F=0.49$.
	}
	\label{fig_orth}
\end{figure*}
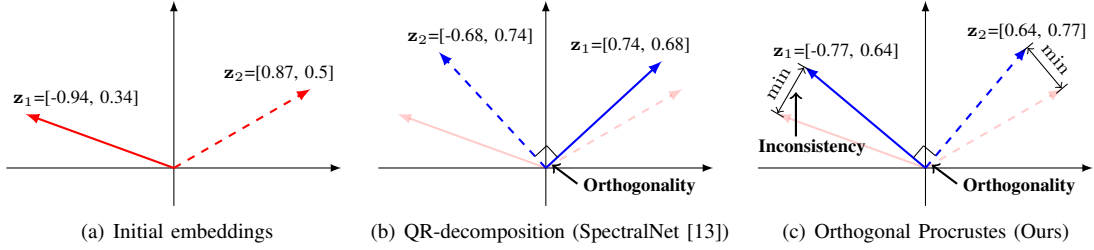

We bootstrap cluster assignments using the soft $k$-means (also referred to as fuzzy clustering \cite{bezdek2013pattern}) objective as a prior, which is defined as: 
\begin{equation}
	\min_{\mathbf{P}\in [0,1]^{N\times K}} \sum_{i=1}^N\sum_{k=1}^{K} P_{ik}\left\|\mathbf{z}_i-\boldsymbol{\mu}_k\right\|^2 ~\text{s.t.}~ \mathbf{P}\boldsymbol{1}_K=\boldsymbol{1}_N.
	\label{eq_km}
\end{equation}
By defining ${H}_{ik}\stackrel{\text{def}}{=}\mathbf{z}_i\boldsymbol{\mu}_k^\intercal$, the above objective can be equivalently written as:
\begin{equation}
	\min_{\mathbf{P}\in [0,1]^{N\times K}}- \sum_{i=1}^N\sum_{k=1}^{K} P_{ik}H_{ik}~\text{s.t.}~ \mathbf{P}\boldsymbol{1}_K=\boldsymbol{1}_N,\frac{K}{N}{\mathbf{P}}^\intercal\boldsymbol{1}_N=\boldsymbol{1}_K,
	\label{eq_p}
\end{equation}
The final constraint encourages equally-sized clusters, where the uniform prior can be replaced with any desired distribution if prior knowledge about cluster sizes is available.
Using the bootstrap procedure described in Section \ref{sec_affinities}, in the target estimation step, we treat Equation \eqref{eq_p} as an optimal transport problem and find the target cluster assignments $\mathbf{P}^+$:
\begin{equation}
	{\mathbf{P}^+}=\operatorname{Diag}(\boldsymbol{\alpha}^*)\exp({\mathbf{H}}/{\eta }) \operatorname{Diag}(\boldsymbol{\beta}^*). 
	\label{eq_pplus}
\end{equation}
In the parameter update step, we minimize the clustering loss:
\begin{equation}
	\mathcal{L}_{{c}}(\mathbf{P}^+,\mathbf{P})=-\sum_{i=1}^N\sum_{k=1}^{K}P^+_{ik}\log P_{ik}.
	\label{loss_cluster}
\end{equation}

\subsection{Orthogonal Re-parameterization Technique}
\label{sec_ort}
The orthogonality of spectral embeddings is important for identifying multiple clusters ($K>2$).
However, standard deep neural networks cannot inherently guarantee orthogonality in their output.
Most existing models map the output $\mathbf{Z}$ into an orthogonalized counterpart $\mathbf{Z}_\text{new}$, computed as  $\mathbf{Z}_\text{new}=\mathbf{Z}\mathbf{R}^{-1}$, where $\mathbf{R}$ is derived from the QR-decomposition of $\mathbf{Z} = \mathbf{Q}\mathbf{R}$ \cite{shaham2018spectralnet}.
However, experiments indicate that this method cannot be applied to end-to-end deep spectral clustering, which is because the semantic inconsistency---defined as $\|\mathbf{Z}-\mathbf{Z}_\text{new}\|_F$ in this paper---is not considered and could be arbitrarily large.
Such large semantic inconsistency leads to training instability in the clustering head (Figure \ref{fig_inconsistency}).

To handle the above issue, we propose a re-parameterization technique that guarantees orthogonality while minimizing the semantic inconsistency (Figure \ref{fig_orth}).
We leverage orthogonal Procrustes \cite{schonemann1966generalized} for orthogonalization, which finds the orthogonal matrix that most closely maps a given matrix to another.
\begin{lemma}[Orthogonal Procrustes \cite{schonemann1966generalized}]\label{op}
	For two matrices $\mathbf{A}$ and  $\mathbf{B}$, with singular value decomposition (SVD) $\mathbf{B}\mathbf{A}^\intercal=\mathbf{U}\mathbf{\Sigma}\mathbf{V}^\intercal$, we have:
	\begin{equation}
		\underset{\mathbf{Q}}{\arg\min} \|\mathbf{Q}\mathbf{A}-\mathbf{B}\|_F^2 =\mathbf{U}\mathbf{V}^\intercal~\text{s.t.}~ \mathbf{Q}^\intercal\mathbf{Q}=\mathbf{I}.
	\end{equation}
\end{lemma}

With this lemma, we propose mapping the identity matrix to any network output, which leads to the following theorem.

\begin{theorem}
	For any network output $\mathbf{Z}$ with SVD $\mathbf{Z}=\mathbf{U}\mathbf{\Sigma}\mathbf{V}^\intercal$, its orthogonalized counterpart with minimal semantic inconsistency is given by $\mathbf{Z}_\text{new}=\mathbf{U}\mathbf{V}^\intercal$.
\end{theorem}
\begin{proof}
	Applying Lemma~\ref{op}, we set $\mathbf{A}$ as the identity matrix $\mathbf{I}$  and $\mathbf{B}$ as the network output $\mathbf{Z}$. According to the lemma, the best orthogonal mapping is given by $\mathbf{U}\mathbf{V}^\intercal$, where $\mathbf{U}$ and $\mathbf{V}$ are derived from the SVD of $\mathbf{B}\mathbf{A}^\intercal=\mathbf{Z}\mathbf{I}^\intercal=\mathbf{Z}$. This mapping ensures orthogonality while minimizing semantic inconsistency. Hence, the orthogonalized counterpart of $\mathbf{Z}$ with minimal semantic inconsistency is $\mathbf{Z}_\text{new}=\mathbf{U}\mathbf{V}^\intercal$.
\end{proof}

With the above theorem, we achieve semantically consistent orthogonalization using SVD.
A problem is that SVD may hinder stochastic gradient descent (SGD) based training as the gradients of singular vectors tend to be numerically unstable.
To address this issue, we re-parametrize spectral embeddings $\mathbf{Z}$ using the straight-through estimator \cite{bengio2013estimating}:
\begin{equation}
	\mathbf{Z}_\text{new}=\mathbf{Z}+\operatorname{sg}(\mathbf{Z}_\text{new}-\mathbf{Z}),
	\label{eq_rep}
\end{equation}
where $\operatorname{sg}(\cdot)$ denotes the stop gradient operation.
We only apply the orthogonal re-parameterization technique during training and omit it during testing as the spectral embeddings learned by BootSC naturally tend to be orthogonal (Figure \ref{fig_inconsistency}).

\begin{algorithm}[t]
\caption{BootSC}
\definecolor{codeblue}{rgb}{0.25,0.5,0.5}
\lstset{
	basicstyle=\fontsize{7.2pt}{7.2pt}\ttfamily\bfseries,
	commentstyle=\fontsize{7.2pt}{7.2pt}\color{codeblue},
	keywordstyle=\fontsize{7.2pt}{7.2pt},
}
\begin{lstlisting}[language=python]
# training
for x in loader: #  load a minibatch data
  x1, x2 = aug(x), aug(x)  # random augmentation
  z1, z2 = f_a(x1), f_a(x2) # spectral embeddings
  z1, z2 = orth_l2(z1), orth_l2(z2)  # re-parametrize
  # predicted and target affinities
  w1, w2 = cross_affinity(z1), cross_affinity(z2)
  w1p, w2p = sinkhorn(w1), sinkhorn(w2)
  # predicted and target cluster assignments
  p1, p2 = f_c(z1), f_c(z2)
  p1p, p2p = sinkhorn(p1), sinkhorn(p2)
  ta, tc = ta_log.exp(), tc_log.exp() # softmax temperature 
  la = cross_entropy(w1/ta, w2p) + cross_entropy(w2/ta, w1p)
  lc = cross_entropy(p1/tc, p2p) + cross_entropy(p2/tc, p1p)
  loss = la + lc*lamb
  loss.backword() # back-propagate
  update(f_a, f_c, ta_log, tc_log) # SGD update

# testing
for x in loader: 
  z = normalize(f_a(x), dim=1) # l2-normalized embeddings
  p = f_c(z).argmax(dim=1) # predicted cluster assignments

# function
def orth_l2(z): # orthogonalize and l2-normalize
  u, _, vh = svd(z) # orthogonal Procrustes
  q = matual(u, vh) 
  z = z + (q - z).detach() # straight-through estimator
  return normalize(z, dim=1) # l2-normalize

def cross_affinity(z): # compute off-diagonal affinity matrix
  w = matmul(z, z.T) 
  bs = z.size(0) # batch size
  diag_mask = (1 - eye(bs)).bool()
  return w[diag_mask].view(bs, -1) # remove diagonal elements

def sinkhorn(w, eta=0.05, isk=5): # sinkhorn's iteration
  w = exp(w.detach()/eta)
  for _ in range(isk):
    w/=sum(w, dim=0, keepdim=True) # column-normalize
    w/=sum(w, dim=1, keepdim=True) # row-normalize
  return w
\end{lstlisting}
\label{algo}
\end{algorithm}

\subsection{End-to-End Joint Training}
By simultaneously minimizing the affinity loss and the clustering loss, we jointly optimize the entire spectral clustering pipeline.
We use mini-batch training for stochastic optimization, which trades off accuracy for scalability.
As training is \textit{from scratch}, we may encounter the representation collapse issue where all samples are mapped into the same embedding.
Similar problems have been well-addressed in contrastive learning, and popular solutions include stop-gradient operation \cite{chen2021exploring}, momentum encoder \cite{he2020momentum}, and swapped prediction \cite{caron2020unsupervised}.
We adopt the swapped prediction for BootSC.
Specifically, given a mini-batch of samples $\mathbf{X}$, we randomly augment $\mathbf{X}$ twice to produce two sets of new samples $\mathbf{X}_1$ and $\mathbf{X}_2$.
Then, the total loss function of BootSC is computed as:
\begin{equation}
	\sum_{\substack{u, v \in \{1, 2\} \\ u \neq v}} \mathcal{L}_{{a}}(\mathbf{W}_u^+, \mathbf{W}_v) + \lambda\mathcal{L}_{{c}}(\mathbf{P}_u^+, \mathbf{P}_v),
	\label{eq_loss}
\end{equation}
where $\lambda$ is the trade-off parameter, and $\mathbf{W}_u, \mathbf{W}_u^+,\mathbf{P}_u,\mathbf{P}_u^+$ are obtained by using $\mathbf{X}_u$ as model input.

Algorithm \ref{algo} provides a PyTorch-like pseudo-code implementation of BootSC for training and testing.
The time complexity of BootSC is $\mathcal{O}(2(BD+B^2D+BD+B^2I_\text{SK}+B^2)+B^2)$  for training and $\mathcal{O}(BK+BKD)$ for testing, where $B$ is the batch size, $K$ is the number of clusters, $D$ is the embedding dimensionality, and $I_\text{SK}$ is the number of iterations in Sinkhorn’s fixed point iteration.

	\section{Experiments}
\noindent\textbf{Datasets.}
We use five benchmark image datasets for evaluation, including: CIFAR-10\cite{krizhevsky2009learning}, CIFAR-100 (20 superclasses) \cite{krizhevsky2009learning}, ImageNet-10 \cite{chang2017deep}, ImageNet-Dogs \cite{chang2017deep}, and Tiny-ImageNet \cite{le2015tiny}. 
Following prior work \cite{metaxas2023divclust}, we utilize both the training and testing sets of CIFAR-10 and CIFAR-100, but only the training set of ImageNet-based datasets.
Table \ref{tab_ds} shows the key statistics of the five datasets.
\begin{table}[htbp]
	\centering
	\caption{Dataset statistics.}
	\begin{tabular}{c|ccc}
		\toprule
		Dataset &Split& \# Sample & \# Cluster\\ \midrule
		CIFAR-10 &Train+Test& 60,000 & 10\\
		CIFAR-100 &Train+Test& 60,000 & 20  \\
		ImageNet-10 &Train& 13,000 & 10  \\
		ImageNet-Dogs &Train& 19,500 & 15  \\
		Tiny-ImageNet &Train&100,000 & 200 \\ \bottomrule
	\end{tabular}
	\label{tab_ds}
\end{table}

\begin{table*}[!t]
	\centering
	\caption{Comparison of clustering results. The dash denotes that the result is unavailable. The best results are highlighted in \textbf{boldface}.}
	\begin{tabular}{c|ccc|ccc|ccc|ccc|ccc}
		\toprule
		\multirow{2}{*}{Method}  & \multicolumn{3}{c|}{CIFAR-10} & \multicolumn{3}{c|}{CIFAR-100} & \multicolumn{3}{c|}{ImageNet-10} & \multicolumn{3}{c|}{ImageNet-Dogs} & \multicolumn{3}{c}{Tiny-ImageNet} \\
		& NMI & ACC & ARI & NMI & ACC & ARI & NMI & ACC & ARI & NMI & ACC & ARI & NMI & ACC & ARI \\ \midrule
		$k$-means \cite{lloyd1982least} & 0.087 & 0.229 & 0.049 & 0.084 & 0.130 & 0.028 & 0.119 & 0.241 & 0.057 & 0.055 & 0.150 & 0.020 & 0.065 & 0.025 & 0.005 \\
		SC \cite{zelnik2004self} & 0.103 & 0.247 & 0.085 & 0.090 & 0.136 & 0.022 & 0.151 & 0.274 & 0.076 & 0.038 & 0.111 & 0.013 & 0.063 & 0.022 & 0.004 \\
		SE-ISR \cite{wang2021fast} &0.084&0.205&0.040&0.080&0.126&0.020&0.152&0.248&0.065&0.028&0.104&0.010&-&-&-\\
		AC \cite{gowda1978agglomerative} & 0.105 & 0.228 & 0.065 & 0.098 & 0.138 & 0.034 & 0.138 & 0.242 & 0.067 & 0.037 & 0.139 & 0.021 & 0.037 & 0.139 & 0.021 \\
		NMF \cite{cai2009locality} & 0.081 & 0.190 & 0.034 & 0.079 & 0.118 & 0.026 & 0.132 & 0.230 & 0.065 & 0.044 & 0.118 & 0.016 & 0.072 & 0.029 & 0.005 \\
		SpectralNet \cite{shaham2018spectralnet} &0.121&0.250&0.062&0.101&0.148&0.034&0.164&0.263&0.075&0.058&0.129&0.021&0.113&0.034&0.006 \\
		DSCCLR \cite{li2024deep} &0.082&0.237&0.059&-&-&-&-&-&-&-&-&-&-&-&-\\
		GCC \cite{zhong2021graph} &0.764 &0.856& 0.728 &0.472 &0.472& 0.305&0.842& 0.901 &0.822&0.490&0.526&0.362&0.347&0.138&0.075\\
		WEC-GAN \cite{CaiZWFG24} &0.334&0.474&-&-&-&-& 0.855 & 0.922 & -& 0.441 & 0.437 & -& 0.383 & 0.177 & -\\
		JULE \cite{yang2016joint} & 0.192 & 0.272 & 0.138 & 0.103 & 0.137 & 0.033 & 0.175 & 0.300 & 0.138 & 0.054 & 0.138 & 0.028 & 0.102 & 0.033 & 0.006 \\
		DEC \cite{xie2016unsupervised} & 0.257 & 0.301 & 0.161 & 0.136 & 0.185 & 0.050 & 0.282 & 0.381 & 0.203 & 0.122 & 0.195 & 0.079 & 0.115 & 0.037 & 0.007 \\
		DAC \cite{chang2017deep} & 0.396 & 0.522 & 0.306 & 0.185 & 0.238 & 0.088 & 0.394 & 0.527 & 0.302 & 0.219 & 0.275 & 0.111 & 0.190 & 0.066 & 0.017 \\
		DCCM \cite{wu2019deep} & 0.496 & 0.623 & 0.408 & 0.285 & 0.327 & 0.173 & 0.608 & 0.710 & 0.555 & 0.321 & 0.383 & 0.182 & 0.224 & 0.108 & 0.038 \\
		PICA \cite{huang2020deep} & 0.591 & 0.696 & 0.512 & 0.310 & 0.337 & 0.171 & 0.802 & 0.870 & 0.761 & 0.352 & 0.352 & 0.201 & 0.277 & 0.098 & 0.040 \\
		HCSC \cite{guo2022hcsc} & 0.407 & 0.480 & 0.295 & 0.361 & 0.362 & 0.206 & 0.647 & 0.741 & 0.559 & 0.355 & 0.355 & 0.209 & 0.305 & 0.139 & 0.006 \\
		IDFD \cite{tao2020clustering} & 0.711 & 0.815 & 0.663 & 0.426 & 0.425 & 0.264 & 0.898 & 0.954 & 0.901 & \multicolumn{1}{c}{0.546} & \multicolumn{1}{c}{0.591} & \multicolumn{1}{c|}{0.413} & - & - & - \\
		SCAN \cite{van2020scan} & 0.787 & 0.876 & 0.758 & 0.468 & 0.459 & 0.301 & - & - & - & - & - & - & - & - & - \\
		CC \cite{li2021contrastive} & 0.705 & 0.790 & 0.637 & 0.431 & 0.429 & 0.266 & 0.859 & 0.893 & 0.822 & 0.445 & 0.429 & 0.274 & 0.340 & 0.140 & 0.071 \\
		DeepCluE \cite{huang2022deepclue} & 0.727 & 0.764 & 0.646 & 0.472 & 0.457 & 0.288 & 0.882 & 0.924 & 0.856 & \multicolumn{1}{c}{0.448} & \multicolumn{1}{c}{0.416} & \multicolumn{1}{c|}{0.273} & 0.379 & 0.194 & 0.102 \\
		SSCN \cite{wang2024semantic} &0.751&0.845&0.707&0.489&0.488&0.328& - & - & -& - & - & -& - & - & -\\
		DivClust \cite{metaxas2023divclust} & 0.710 & 0.815 & 0.675 & 0.440 & 0.437 & 0.283 & 0.850 & 0.900 & 0.819 & 0.516 & 0.529 & 0.376 &-  & - &-  \\ 
		\methodname (Ours) &\textbf{0.814}  & \textbf{0.883} & \textbf{0.780} & \textbf{0.528} & \textbf{0.492} & \textbf{0.347} & \textbf{0.913} & \textbf{0.962} & \textbf{0.918} & \textbf{0.596} & \textbf{0.643} & \textbf{0.473} & \textbf{0.398} & \textbf{0.199} & \textbf{0.111} \\ \bottomrule
	\end{tabular}
	\label{tab_sota}
\end{table*}

\begin{figure*}[thbp]
	\centering	
	\newcommand{\AffinityGraphWithd}{0.32}
	\subfigure{\includegraphics[width=\AffinityGraphWithd\linewidth]{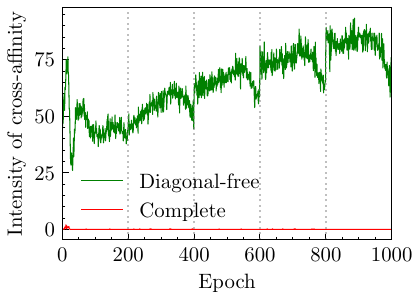}}
	\subfigure{\includegraphics[width=\AffinityGraphWithd\linewidth]{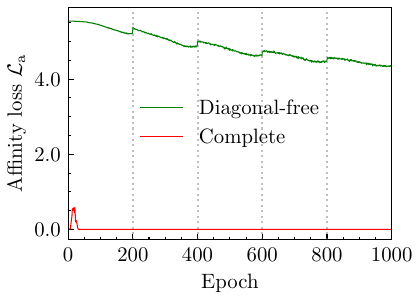}}
	\subfigure{\includegraphics[width=\AffinityGraphWithd\linewidth]{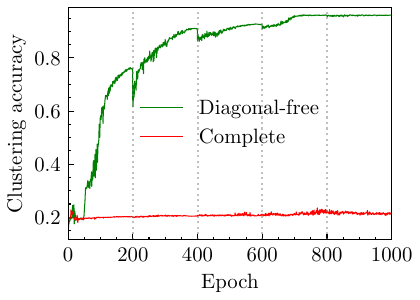}}
	\caption{Avoiding trivial solution. The gray dotted lines indicate learning rate restarts. The original \textit{complete} objective of spectral clustering results in degenerate clustering, while our \textit{diagonal-free} modification produces meaningful clustering.}
	\label{fig_trivialsolution}
\end{figure*}

\noindent\textbf{Implementation Details.}
We use the original $32\times32$ pixel images as input for the CIFAR-10 and CIFAR-100 datasets.
For the ImageNet-based datasets with more complex patterns, we resize the images to $64\times64$ to preserve more visual details, except for the Tiny-ImageNet dataset (still $32\times32$) due to the GPU memory limitation.
For fair comparisons, we employ the same data augmentation scheme as described in prior work \cite{metaxas2023divclust}. 
The batch size $B$ is set to $256$ for each dataset, except for the Tiny-ImageNet dataset ($B=1024$) due to its increased number of clusters.

For fair comparisons with prior works \cite{metaxas2023divclust,li2021contrastive,huang2022deepclue}, we employ the same ResNet-34 \cite{he2016deep} backbone (without the final classification layer) followed by a multilayer perceptron (MLP) projector. 
The MLP projector consists of three fully connected layers with dimensions $[512, 4096, 128]$, which projects the backbone outputs to $128$-dimensional embeddings.
Following the MLP projector is the clustering head, which is a linear layer with an output dimension of $K$.
We set $K$ to the number of ground-truth image categories of each dataset for evaluation.
We train all network parameters from scratch and initialize these parameters using the Xavier approach \cite{glorot2010understanding}.

To ensure fairness and avoid using ground-truth labels for hyperparameter tuning, we fix the following hyperparameters.
The trade-off parameter $\eta$ in both Equations \eqref{wq_wplus} and \eqref{eq_pplus} is set to $0.05$.
The number of iterations $I_{SK}$ in Sinkhorn's fixed point iteration is set to $5$.
The trade-off parameter $\lambda$ in Equation \eqref{eq_loss} is set to $1$.
The logarithmic temperature parameter $\log(\tau)$ is initialized to $\log 0.05$ and constrained to a maximum value of $\log 1$ during training.
The effects of the above hyperparameters on clustering performance are discussed later.

We employ the SGD optimizer with a learning rate of $0.04\times{B}/{256}$, a weight decay of $0.0005$, and a momentum of $0.9$. 
The learning rate is decayed using the cosine decay schedule with restarts every $200$ epochs \cite{loshchilov2016sgdr}. 
We train the proposed \methodname model for $1000$ epochs, which is identical with prior works \cite{metaxas2023divclust,li2021contrastive,huang2022deepclue}.
Our code is based on the PyTorch \cite{paszke2019pytorch} toolbox and publicly available.
We conduct all experiments on a server equipped with one Nvidia GeForce RTX 4090 GPU (GPU memory = $24$GB).
The running time of training the \methodname model is roughly $9$ gpu-hours on the CIFAR-10 dataset; $9$ gpu-hours on CIFAR-100; $7$ gpu-hours on ImageNet-10; $10$ gpu-hours on ImageNet-dogs; and $14$ gpu-hours on Tiny-ImageNet.

\noindent\textbf{Evaluation Metrics.}
We  quantify clustering performance using three standard metrics: (1) Normalized Mutual Information ({NMI})  \cite{vinh2009information}, clustering ACCuracy ({ACC}), and Adjusted Random Index ({ARI}) \cite{hubert1985comparing}.
The three metrics range from zero to one, and higher values indicate better clustering performance. 

\subsection{Comparisons with State-of-the-Art Methods}
We compare the proposed \methodname with 5 traditional clustering methods and 16 deep clustering methods. 
The traditional clustering methods include: $k$-means \cite{lloyd1982least},
SC \cite{zelnik2004self},  SE-ISR \cite{wang2021fast}, AC \cite{gowda1978agglomerative}, and NMF \cite{cai2009locality}.
The deep clustering methods include:
SpectralNet \cite{shaham2018spectralnet}, DSCCLR \cite{li2024deep}, GCC \cite{zhong2021graph}, WEC-GAN \cite{CaiZWFG24},
JULE \cite{yang2016joint}, DEC \cite{xie2016unsupervised}, DAC \cite{chang2017deep}, DCCM \cite{wu2019deep}, PICA \cite{huang2020deep}, HCSC \cite{guo2022hcsc}, IDFD \cite{tao2020clustering}, SCAN \cite{van2020scan}, CC \cite{li2021contrastive}, DeepCluE \cite{huang2022deepclue}, SSCN \cite{wang2024semantic}, and DivClust \cite{metaxas2023divclust}.
The clustering results of the competitors are either borrowed directly from their respective papers or obtained by running their released code with their default configurations.

Table \ref{tab_sota} reports the clustering results of each method on the five benchmark image datasets.
\methodname consistently outperforms the competitors across all datasets.
\methodname significantly outperforms all traditional clustering methods, owing to the powerful representation capabilities of deep neural networks.
Moreover, compared with the deep clustering models, \methodname also achieves more promising results. 
For example, \methodname substantially surpasses the runner-up SSCN on the CIFAR-100 dataset in terms of the NMI metric, 0.528 \textit{vs.} 0.489. 
This is because these competitors rely on Euclidean distance, whereas \methodname effectively exploits data similarities to achieve better clustering.
Additionally, on the highly challenging ImageNet-Dogs dataset, \methodname accomplishes a notable 16\% improvement in NMI over the most competitive baseline.
Overall, these strong results demonstrate the effectiveness of BootSC.

\subsection{Ablation Study and Parameter Sensitivity Analysis}
\noindent\textbf{Trivial Solution.}
To avoid the trivial solution to Equation \eqref{eq_wzzt}, we introduce a diagonal-free modification that ignores self-affinities and only focus on cross-affinities.
To evaluate this modification, we compare it with the original complete objective by analyzing how the intensity of cross-affinity---defined as the sum of off-diagonal elements in the solution $\mathbf{W}^+$, the affinity loss $\mathcal{L}_{{a}}$, and clustering accuracy vary during training on the ImageNet-10 dataset.

Figure \ref{fig_trivialsolution} shows that the intensity of cross-affinity of the original complete objective remains near zero throughout the training process.
This indicates that the model overwhelmingly focuses on capturing trivial intra-sample relationships rather than exploring meaningful inter-sample ones.
Consequently, the affinity loss $\mathcal{L}_{{a}}$ rapidly converges to zero, leading to degenerate clustering results.
In contrast,  our diagonal-free modification encourages nonzero cross-affinities, promoting the model to magnify the affinities between similar samples while diminishing those of dissimilar ones.
Thus, $\mathcal{L}_{{a}}$ decreases steadily and the model identifies meaningful clusters.

\begin{figure}[thbp]
	\newcommand{\Withd}{0.49}
	\centering	
	\subfigure[CIFAR-10]{\includegraphics[width=\Withd\columnwidth]{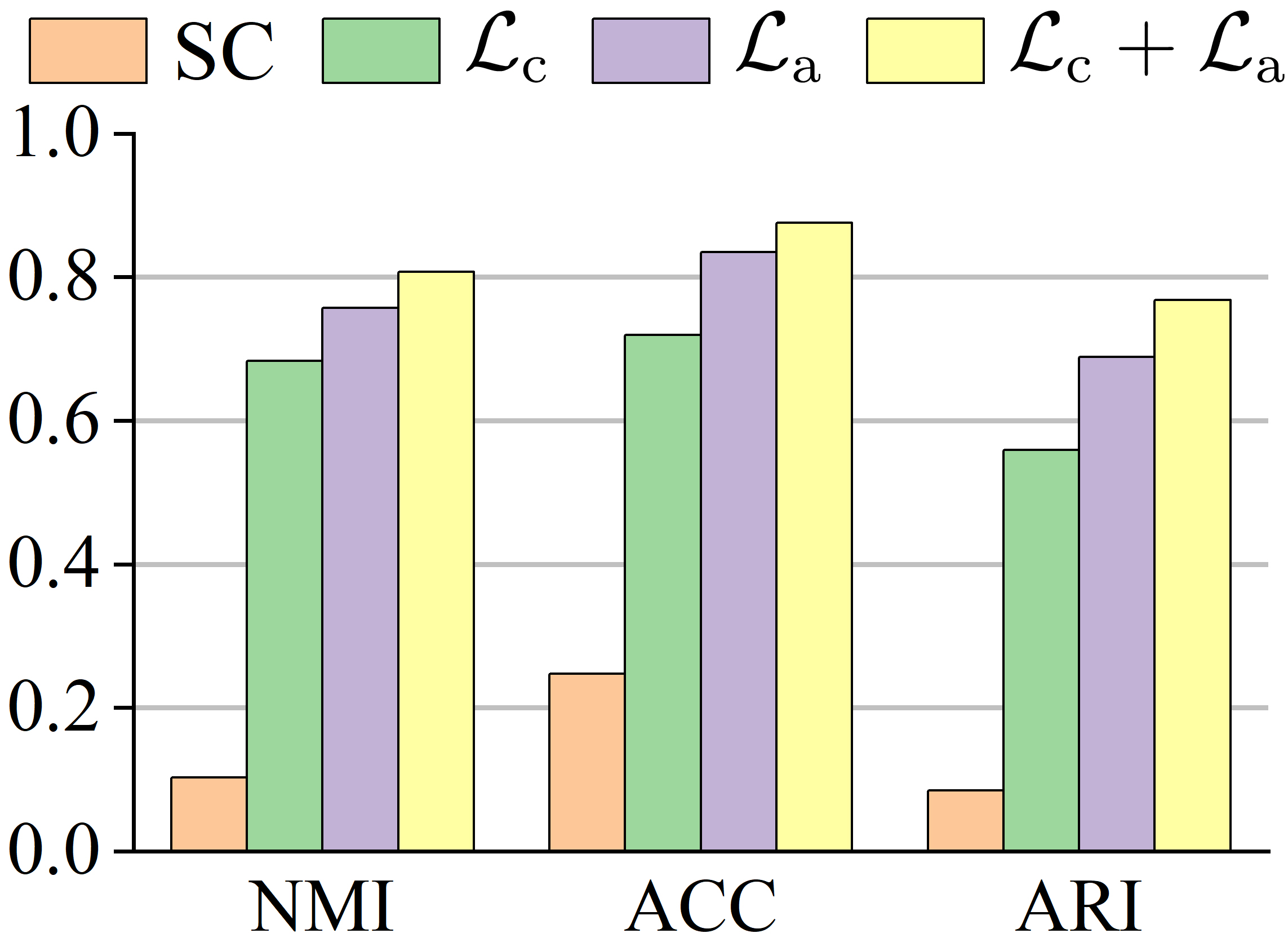}}
	\subfigure[CIFAR-100]{\includegraphics[width=\Withd\columnwidth]{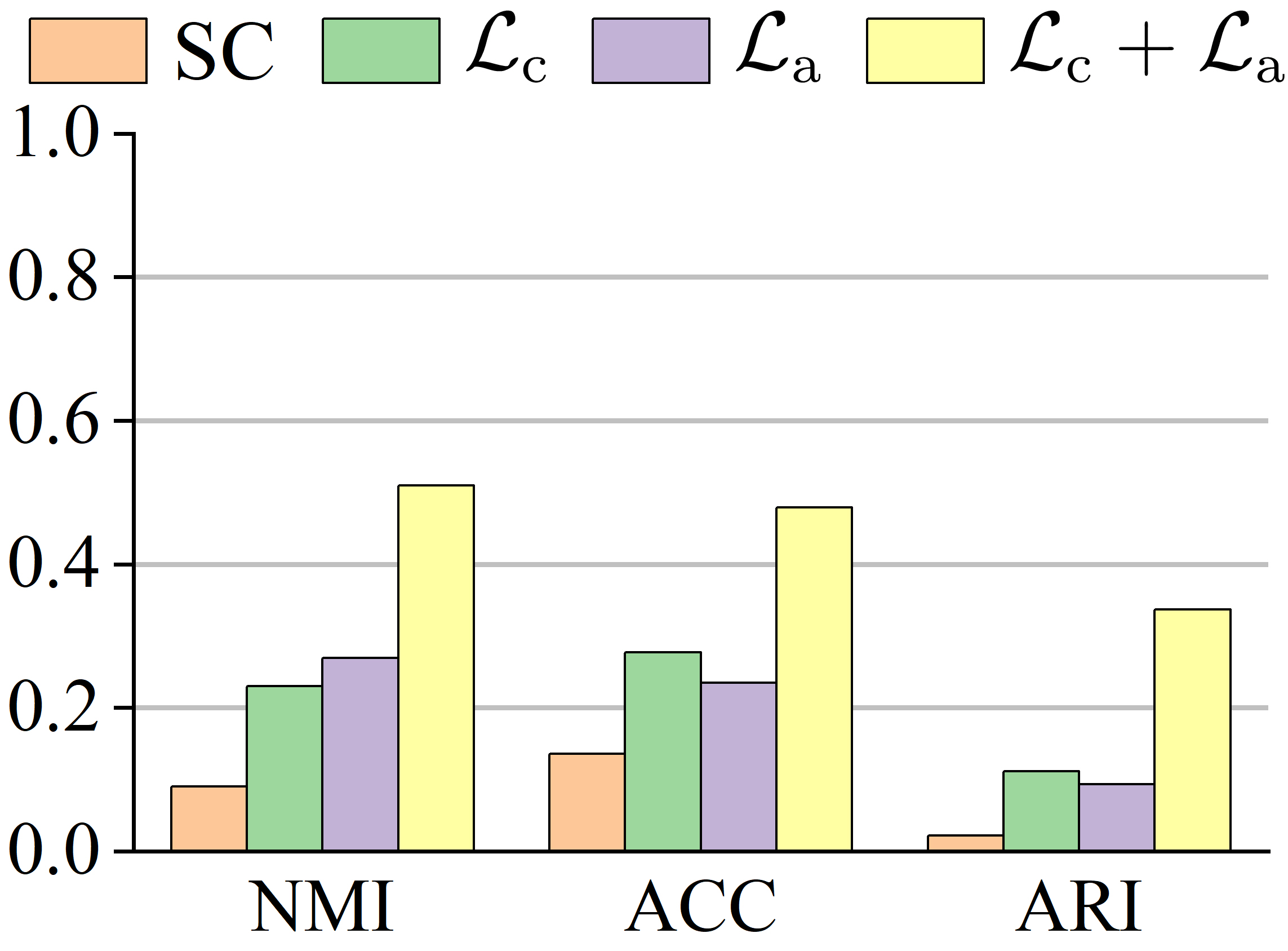}}
	\subfigure[ImageNet-10]{\includegraphics[width=\Withd\columnwidth]{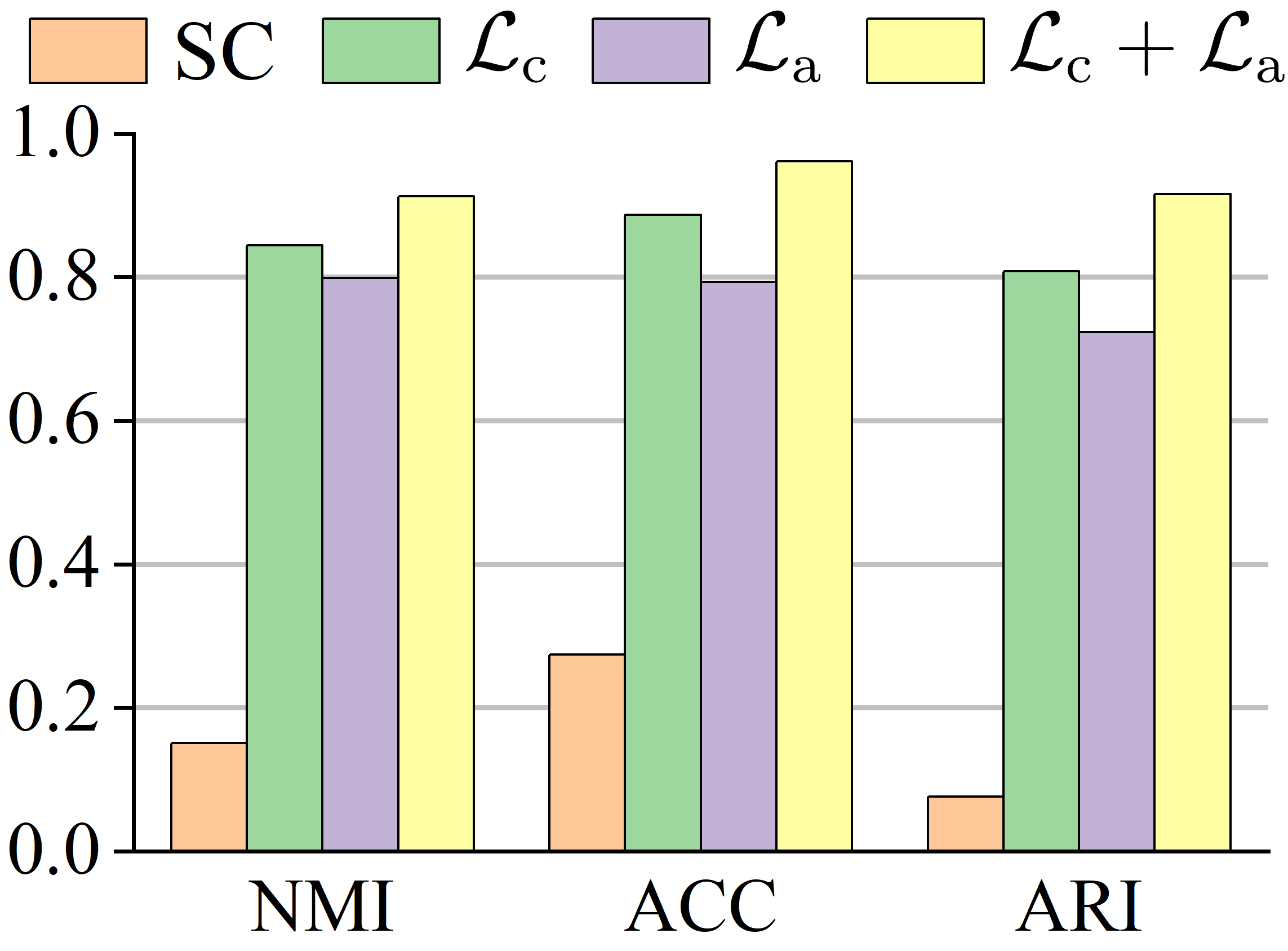}}
	\subfigure[ImageNet-Dogs]{\includegraphics[width=\Withd\columnwidth]{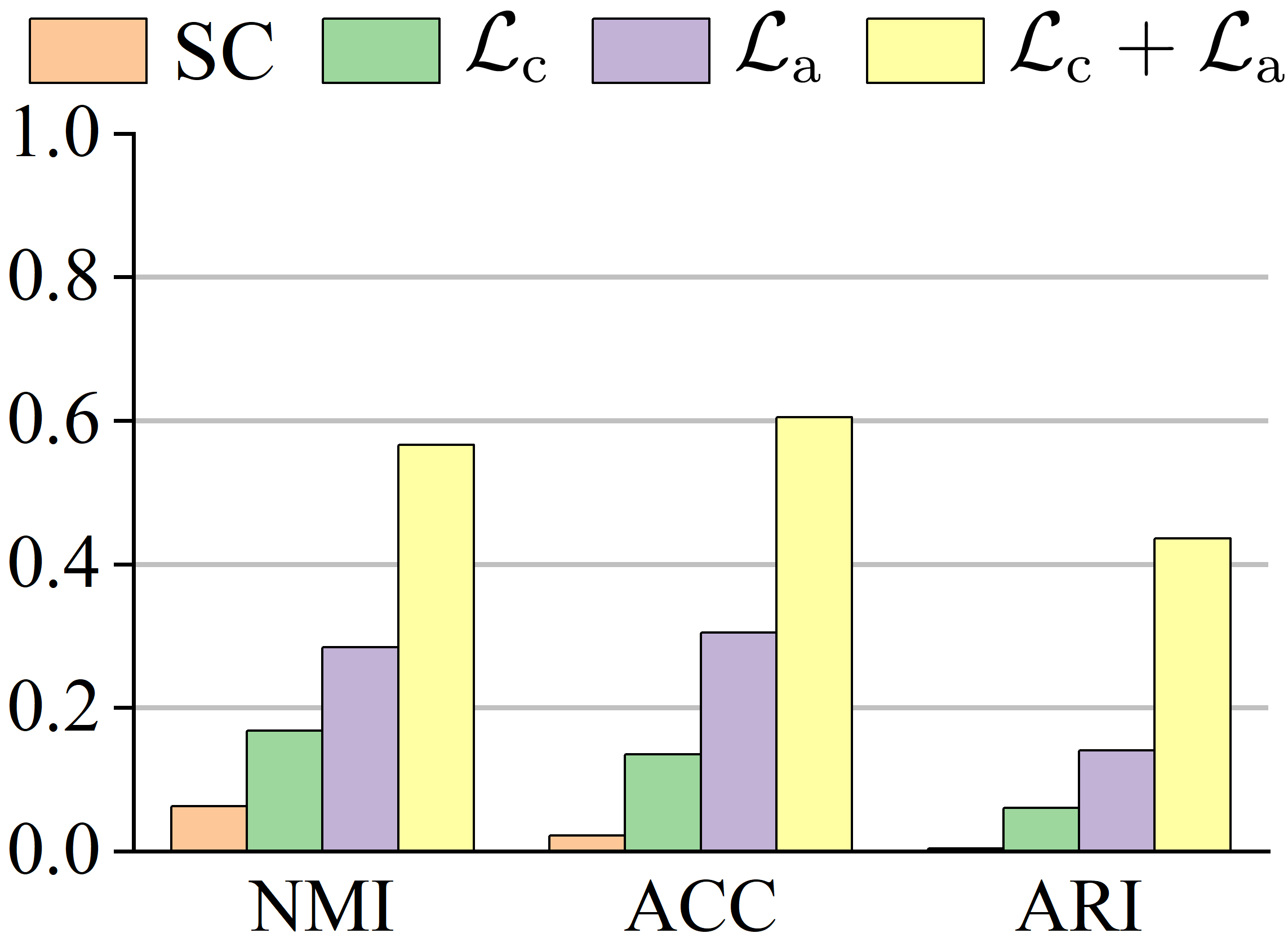}}
	\caption{Ablation study of the loss function.}
	\label{tab_loss}
\end{figure}

\begin{figure}[thbp]
	\newcommand{\Withd}{0.49}
	\centering	
	\subfigure[CIFAR-10]{\includegraphics[width=\Withd\columnwidth]{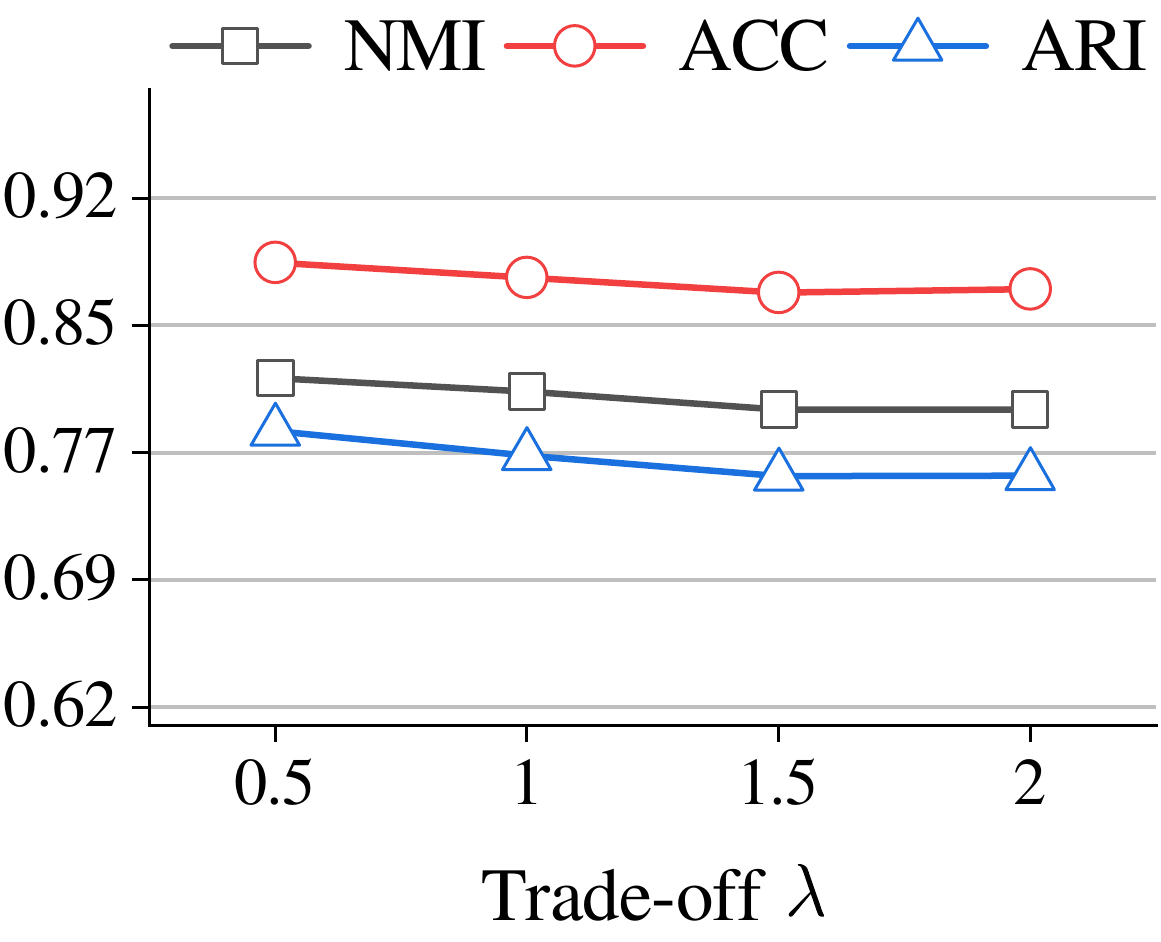}}
	\subfigure[CIFAR-100]{\includegraphics[width=\Withd\columnwidth]{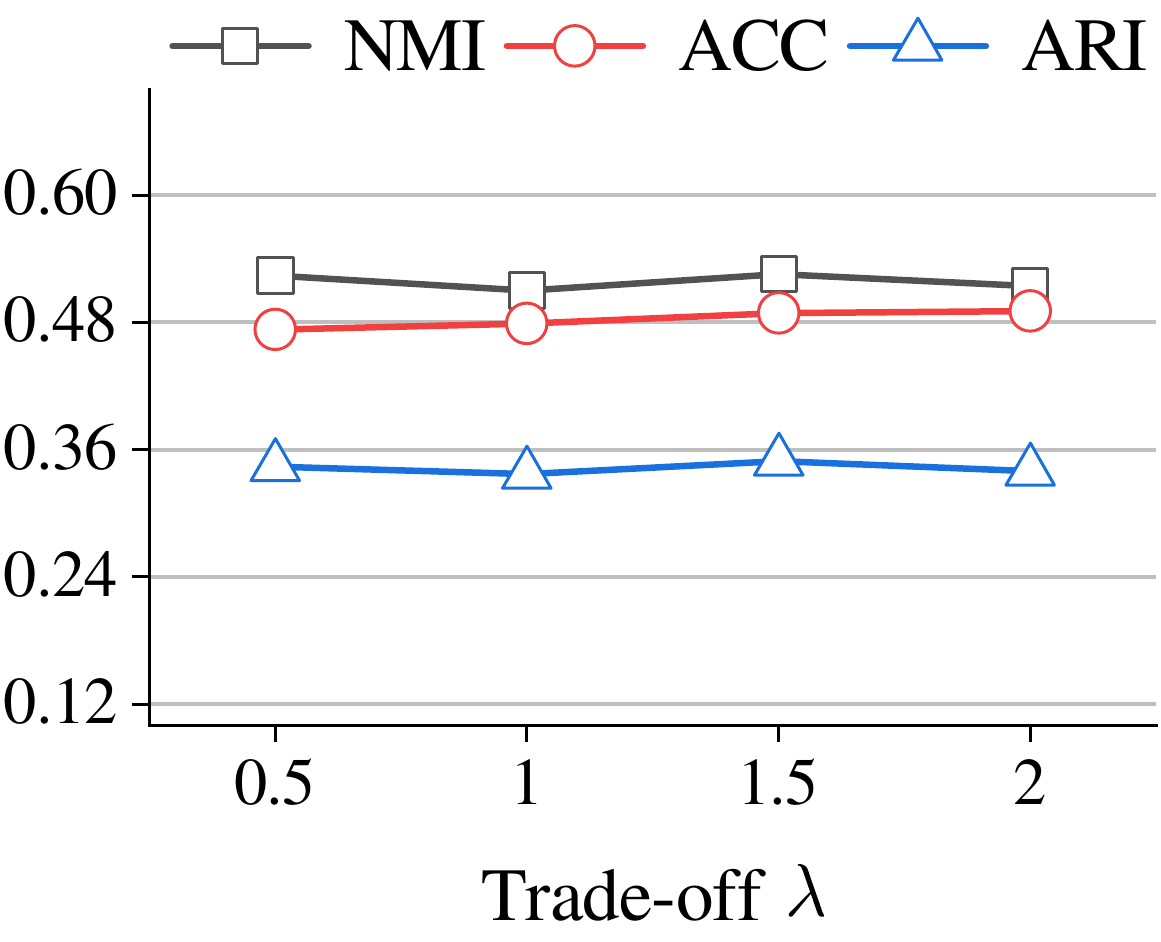}}
	\subfigure[ImageNet-10]{\includegraphics[width=\Withd\columnwidth]{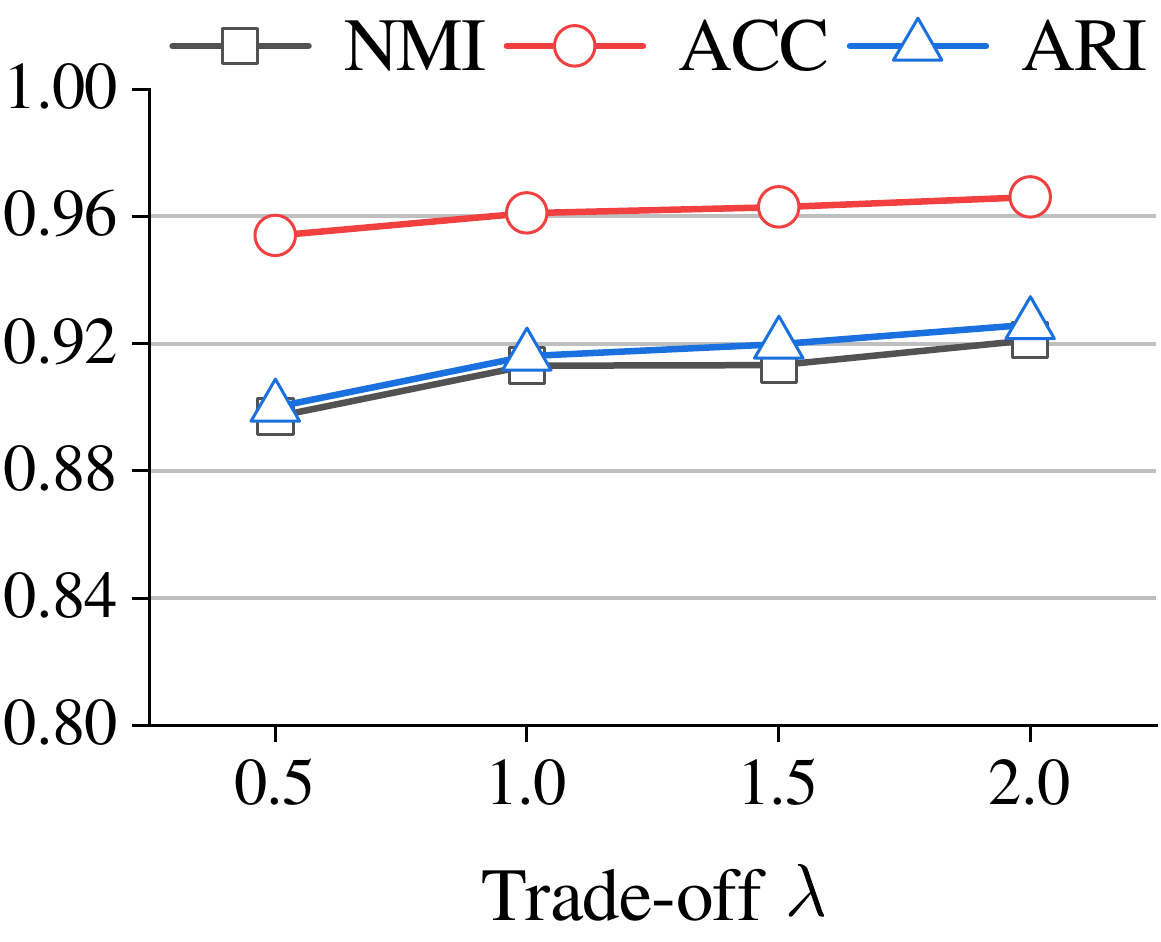}}
	\subfigure[ImageNet-Dogs]{\includegraphics[width=\Withd\columnwidth]{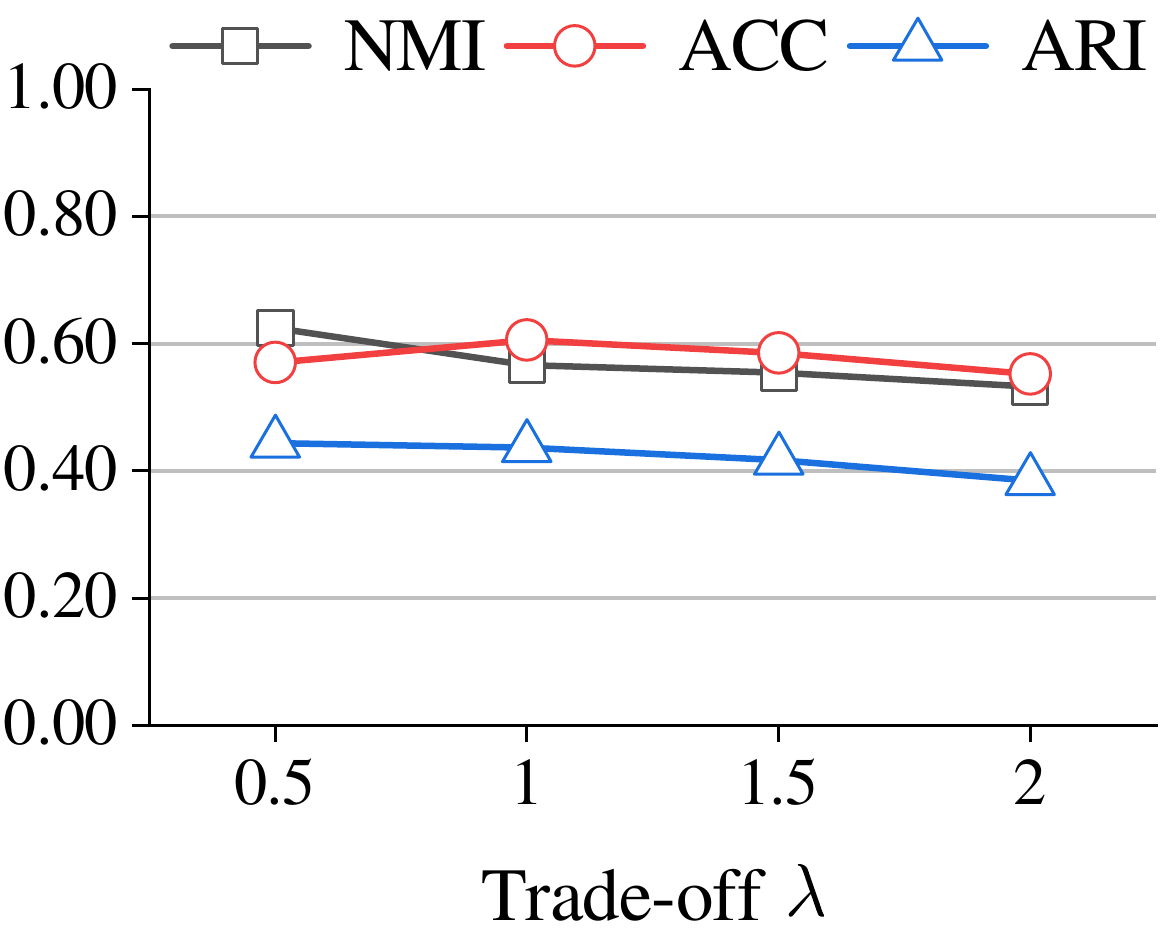}}
	\caption{Effect of the trade-off parameter $\lambda$ in the total loss function $\mathcal{L}_{{a}} + \lambda \mathcal{L}_{{c}}$.}
	\label{tab_lamb}
\end{figure}

\noindent\textbf{Loss Function Analysis.}
\methodname comprises two loss functions:  the affinity loss $\mathcal{L}_{{a}}$ and the clustering loss $\mathcal{L}_{{c}}$.
To investigate the role of each loss function, we first establish a baseline by performing vanilla spectral clustering (SC) \cite{zelnik2004self} on the raw data.
Subsequently,  we train a model by only minimizing $\mathcal{L}_{\text {c}}$ and another model by only minimizing $\mathcal{L}_{\text {a}}$.
Without $\mathcal{L}_{c}$, we cannot directly obtain the cluster assignments and thus conduct posthoc $k$-means clustering on the learned embeddings for computing performance metrics.

\begin{figure}[thbp]
	\newcommand{\Withd}{0.49}
	\centering	
	\subfigure[CIFAR-10]{\includegraphics[width=\Withd\columnwidth]{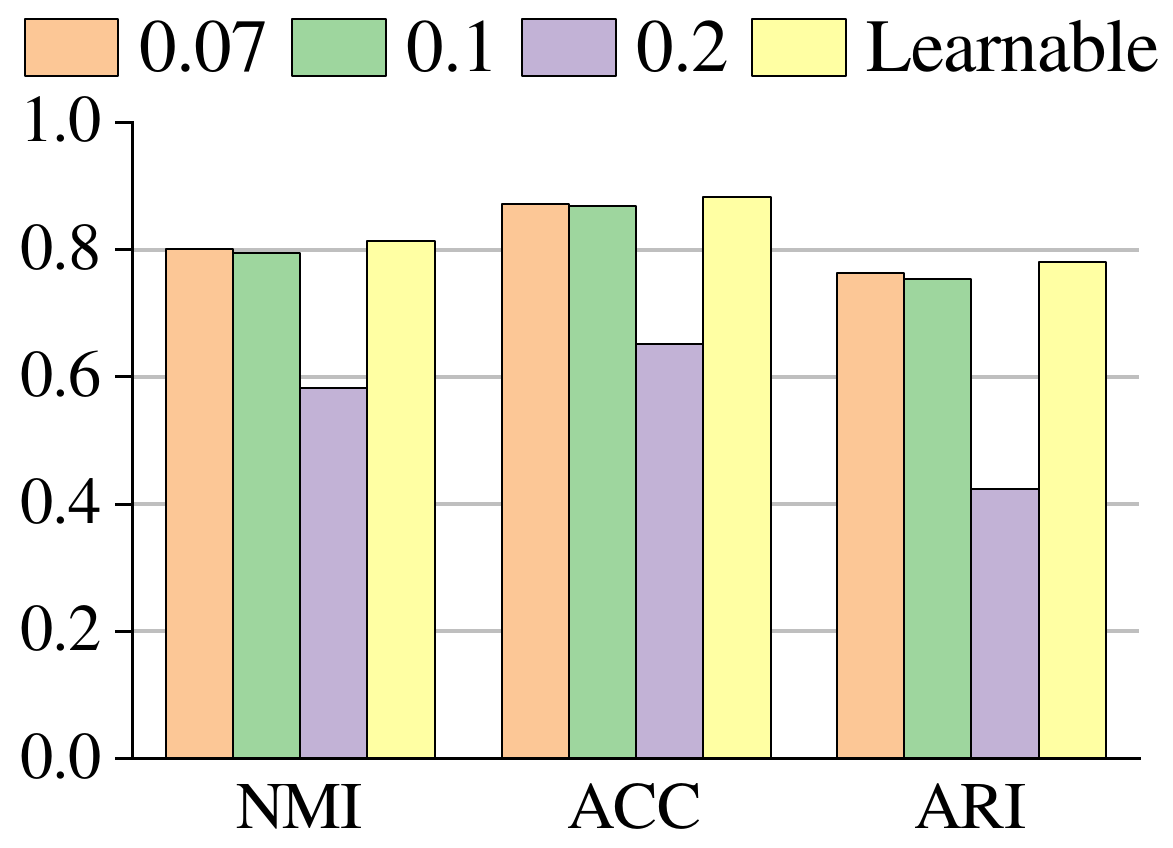}}
	\subfigure[CIFAR-100]{\includegraphics[width=\Withd\columnwidth]{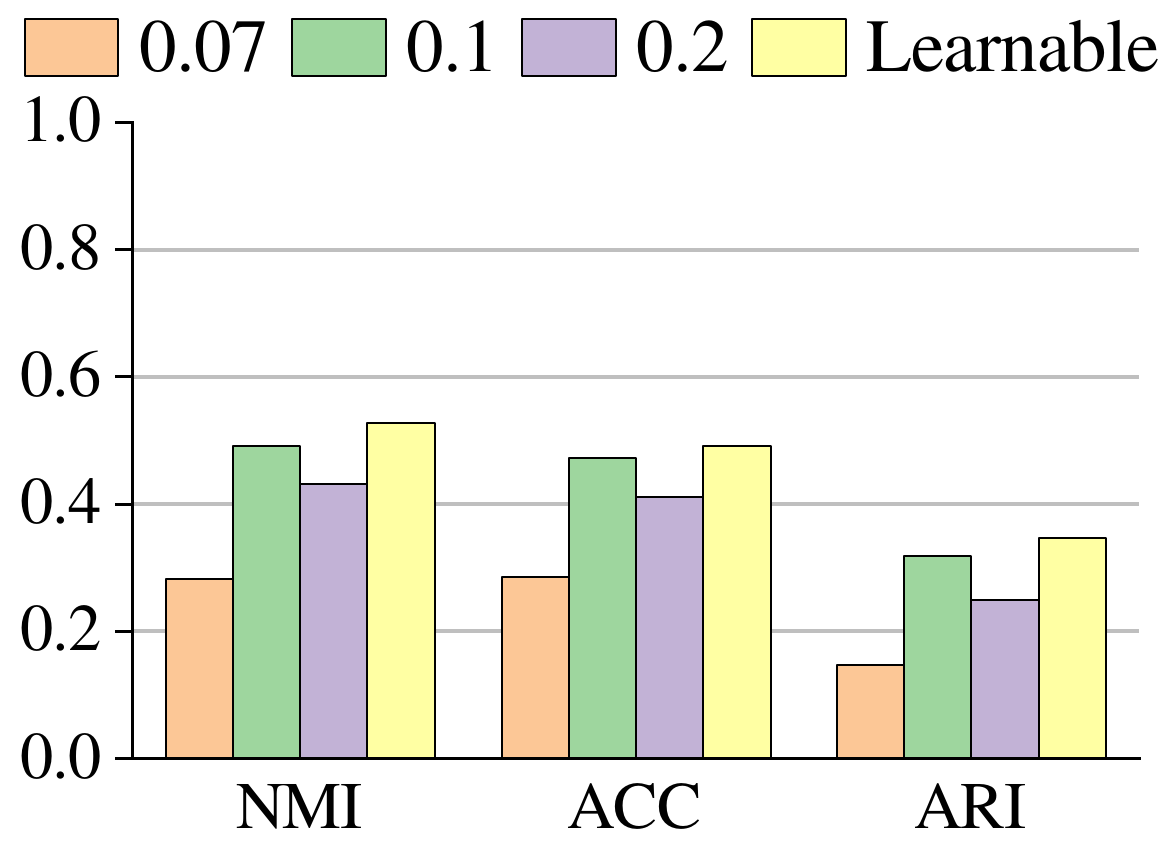}}
	\subfigure[ImageNet-10]{\includegraphics[width=\Withd\columnwidth]{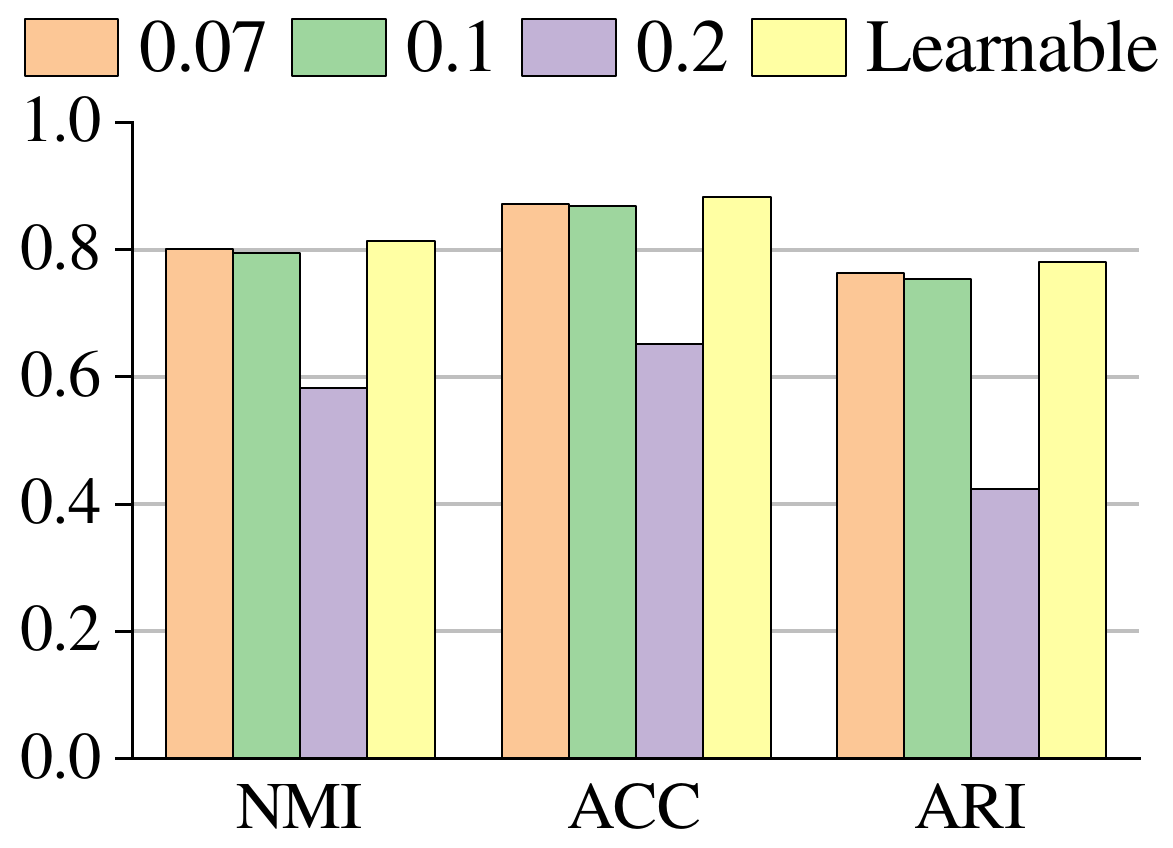}}
	\subfigure[ImageNet-Dogs]{\includegraphics[width=\Withd\columnwidth]{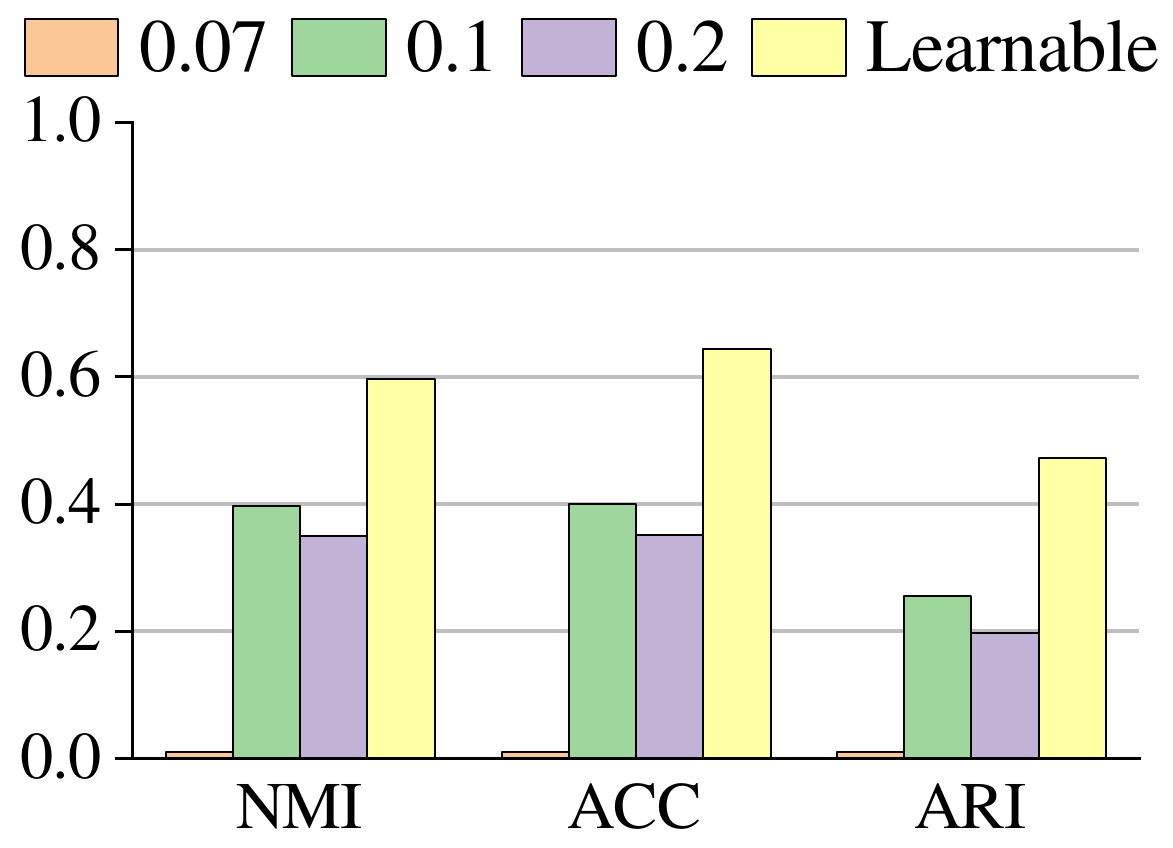}}
	\caption{Fixed \textit{vs.} Learnable temperature parameter $\tau$.}
	\label{tab_tau}
\end{figure}

The results in Figure \ref{tab_loss} underscore the individual contributions of each loss function.
The vanilla spectral clustering yields unsatisfactory results due to its limited representation capability.
Employing $\mathcal{L}_{{c}}$ in isolation yields underperformed results, as it falls under the category of Euclidean distance-based clustering methods that fail to capture the intricate relationships within the data.
While utilizing $\mathcal{L}_{{a}}$ alone enables connectivity-based clustering, it still produces suboptimal results. 
This limitation can be attributed to the absence of joint optimization between embedding learning and clustering. 
Most importantly, the joint optimization of $\mathcal{L}_{{c}}$ and $\mathcal{L}_{{a}}$ achieves the best performance, demonstrating their combination benefit.

We further evaluate the impact of varying the trade-off parameter $\lambda \in [0.5, 1, 1.5, 2]$ in the total loss function $ \mathcal{L}_{{a}} + \lambda \mathcal{L}_{{c}}$, and the results are shown in Figure \ref{tab_lamb}.
The CIFAR-10 dataset slightly favors a lower $\lambda$, while the ImageNet-10 dataset marginally benefits from a larger $\lambda$.
Such variations in preference are expected in end-to-end deep spectral clustering models, as they must balance the two different subtasks (spectral embedding learning and $k$-means clustering).
Notably, all performance metrics exhibit minimal fluctuation across different datasets when $\lambda$ is close to one.
We recommend fixing $\lambda=1$ for different datasets as cross-validation is infeasible for real clustering tasks, which consistently achieves near-optimal performance.

\noindent\textbf{Learnable Temperature Parameter.}
The temperature parameter $\tau$ is a key parameter in spectral clustering.
Finding the optimal value for $\tau$ can be challenging as it varies across different datasets.
We learn $\tau$ during training rather than relying on manual hyperparameter tuning.
To validate this strategy, we compare the learnable $\tau$ against a set of carefully chosen fixed values $[0.07,0.1,0.2]$.
As shown in Figure \ref{tab_tau}, the optimal value of $\tau$ varies across datasets.
For example, setting $\tau=0.07$ yields favorable results on the ImageNet-10 dataset, while resulting in degenerate clustering on the ImageNet-Dogs dataset.
In contrast, the learnable $\tau$ consistently outperforms these fixed settings across all datasets, highlighting the advantages of the learnable strategy.

\begin{figure*}[thbp]
	\centering	
	\includegraphics[width=\linewidth]{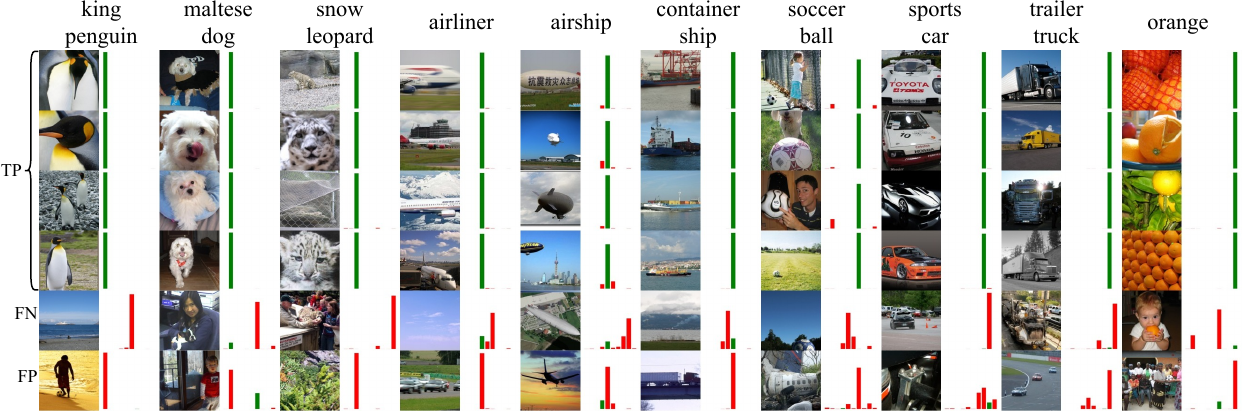}
	\caption{Probabilistic cluster assignments predicted by BootSC on the ImageNet-10 dataset. The top four rows are true positives (\textit{TP}), while the bottom two rows are false negatives (\textit{FN}) and false positives (\textit{FP}), respectively. Accompanying each image is a histogram on the right that indicates the probabilistic cluster assignments, with green bars representing ground-truth categories and red bars representing incorrect ones. All images are selected randomly.}
	\label{img_casestudy}
\end{figure*}

\begin{figure*}[t]
	\centering	
	\newcommand{\AffinityGraphWithd}{0.32}
	\subfigure{\includegraphics[width=\AffinityGraphWithd\linewidth]{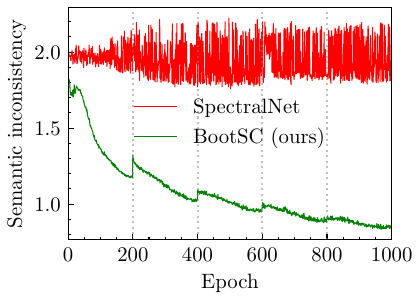}}
	\subfigure{\includegraphics[width=\AffinityGraphWithd\linewidth]{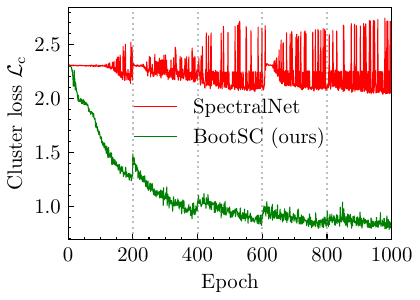}}
	\subfigure{\includegraphics[width=\AffinityGraphWithd\linewidth]{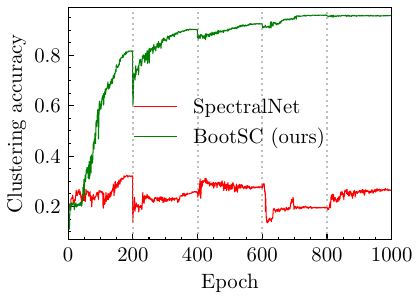}}
	\caption{Existing orthogonalization method \cite{shaham2018spectralnet} encounters training failure while our semantically-consistent orthogonalization addresses this issue. The gray dotted lines indicate where the learning rate restarts.}
	\label{fig_inconsistency}
\end{figure*} 

\begin{table*}[thbp]
	\centering
	\caption{Orthogonalizing spectral embeddings using different methods.}
	\begin{tabular}{c|ccc|ccc|ccc|ccc}
		\toprule
		\multirow{2}{*}{Orthogonalization method}  & \multicolumn{3}{c|}{CIFAR-10} & \multicolumn{3}{c|}{CIFAR-100} & \multicolumn{3}{c|}{ImageNet-10} & \multicolumn{3}{c}{ImageNet-Dogs} \\
		& NMI & ACC & ARI & NMI & ACC & ARI & NMI & ACC & ARI & NMI & ACC & ARI \\ \midrule
		\ding{55} & 0.796 &  0.867& 0.749 &0.524&0.490&0.339& 0.827 & 0.866 & 0.763 &0.514&0.499&0.350\\
		IDFO \cite{tao2020clustering} ($\rho =0.5$) & 0.743& 0.778 & 0.656 &0.513&0.468&0.338& 0.833 &0.874  &0.782  &0.439&0.433&0.282\\
		IDFO \cite{tao2020clustering} ($\rho =1$) & 0.760 &0.787  & 0.678 &0.514&0.499&0.337& 0.831 & 0.871 &0.778  &0.452&0.428&0.277\\
		IDFO \cite{tao2020clustering} ($\rho =2$) & 0.761 & 0.791 & 0.682 &0.515&0.489&0.343& 0.833 &0.870  & 0.777&0.311&0.471&0.311 \\
		QR-decomposition (SpectralNet \cite{shaham2018spectralnet}) & 0.243 &  0.280 & 0.160 &0.160&0.142&0.036&0.282  &0.265  &0.142&0.102&0.134& 0.029 \\
		Orthogonal Procrustes (Ours) &\textbf{0.814}  & \textbf{0.883} & \textbf{0.780} & \textbf{0.528} & \textbf{0.492} & \textbf{0.347} & \textbf{0.913} & \textbf{0.962} & \textbf{0.918} & \textbf{0.596} & \textbf{0.643} & \textbf{0.473} \\ \bottomrule
	\end{tabular}
	\label{tab_orth}
\end{table*}

\begin{table*}[thbp]
	\centering
	\caption{Solving the optimal affinity matrix $\mathbf{W}^+$ using different methods.}
	\begin{tabular}{c|ccc|ccc|ccc|ccc}
		\toprule
		\multirow{2}{*}{Method}  & \multicolumn{3}{c|}{CIFAR-10} & \multicolumn{3}{c|}{CIFAR-100} & \multicolumn{3}{c|}{ImageNet-10} & \multicolumn{3}{c}{ImageNet-Dogs} \\
		& NMI & ACC & ARI & NMI & ACC & ARI & NMI & ACC & ARI & NMI & ACC & ARI \\ \midrule
		Network simplex \cite{bonneel2011displacement} & 0.741  & 0.821 & 0.683 &0.519&0.485&0.342&0.851  & 0.894 &0.820 &0.545&0.578&0.412 \\
		Sinkhorn's iteration \cite{cuturi2013sinkhorn} &\textbf{0.814}  & \textbf{0.883} & \textbf{0.780} & \textbf{0.528} & \textbf{0.492} & \textbf{0.347} & \textbf{0.913} & \textbf{0.962} & \textbf{0.918} & \textbf{0.596} & \textbf{0.643} & \textbf{0.473} \\ \bottomrule  
	\end{tabular}
	\label{tab_eta}
\end{table*}

\begin{figure}[thbp]
	\centering	
	\subfigure{\includegraphics[width=0.45\textwidth]{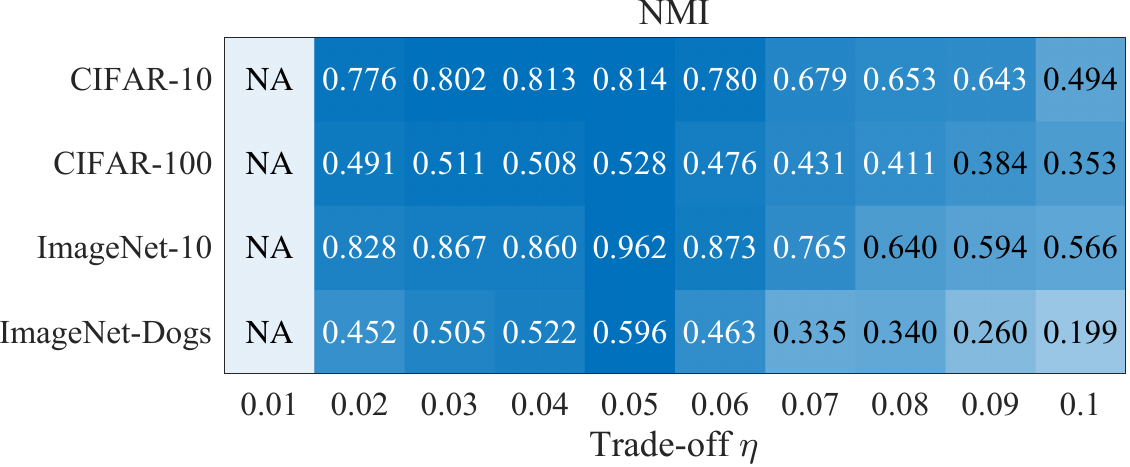}}
	\subfigure{\includegraphics[width=0.45\textwidth]{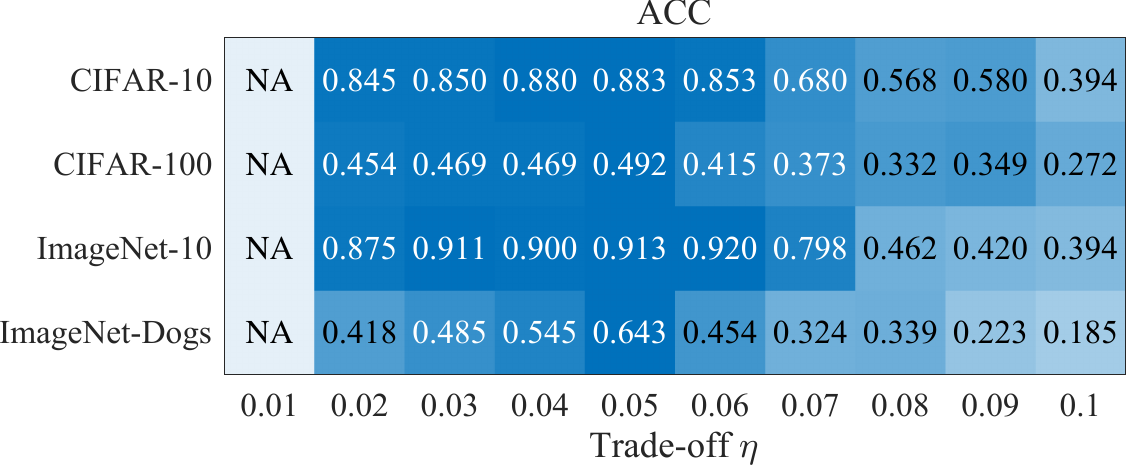}}
	\subfigure{\includegraphics[width=0.45\textwidth]{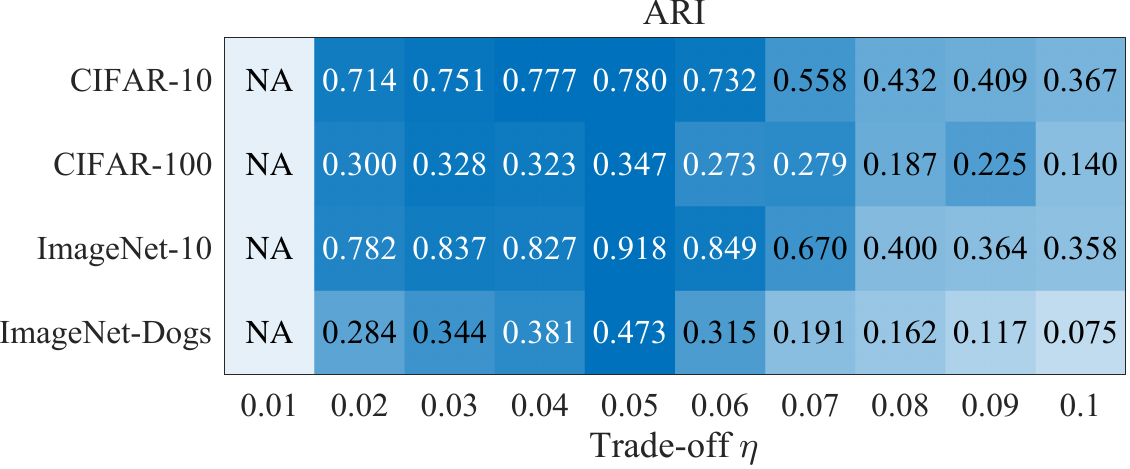}}
	\caption{Effect of the trade-off parameter $\eta$ in Sinkhorn's fixed point iteration. \textit{NA}: the result is not available due to training failure.}
	\label{fig_eta}
\end{figure}

\begin{figure}[thbp]
	\newcommand{\Withd}{0.49}
	\centering	
	\subfigure[CIFAR-10]{\includegraphics[width=\Withd\columnwidth]{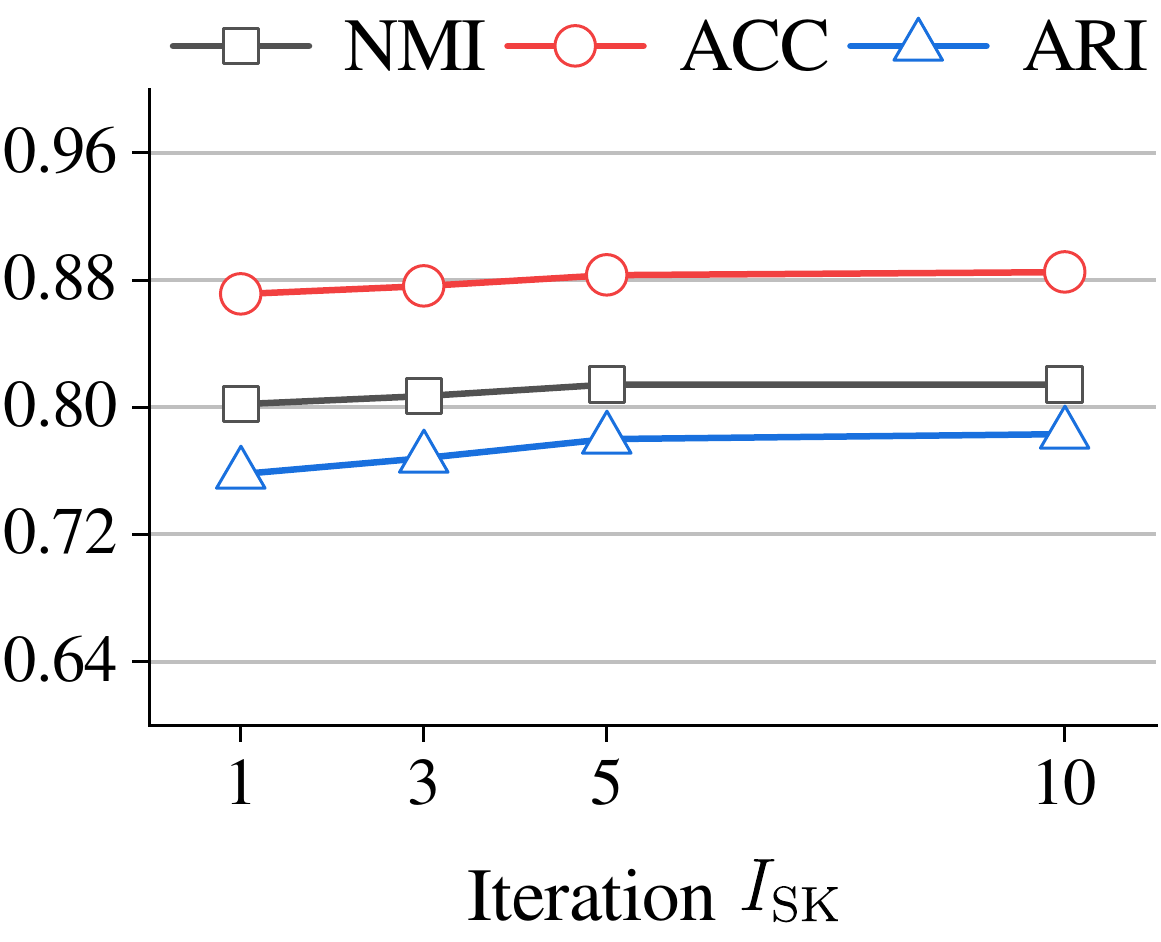}}
	\subfigure[CIFAR-100]{\includegraphics[width=\Withd\columnwidth]{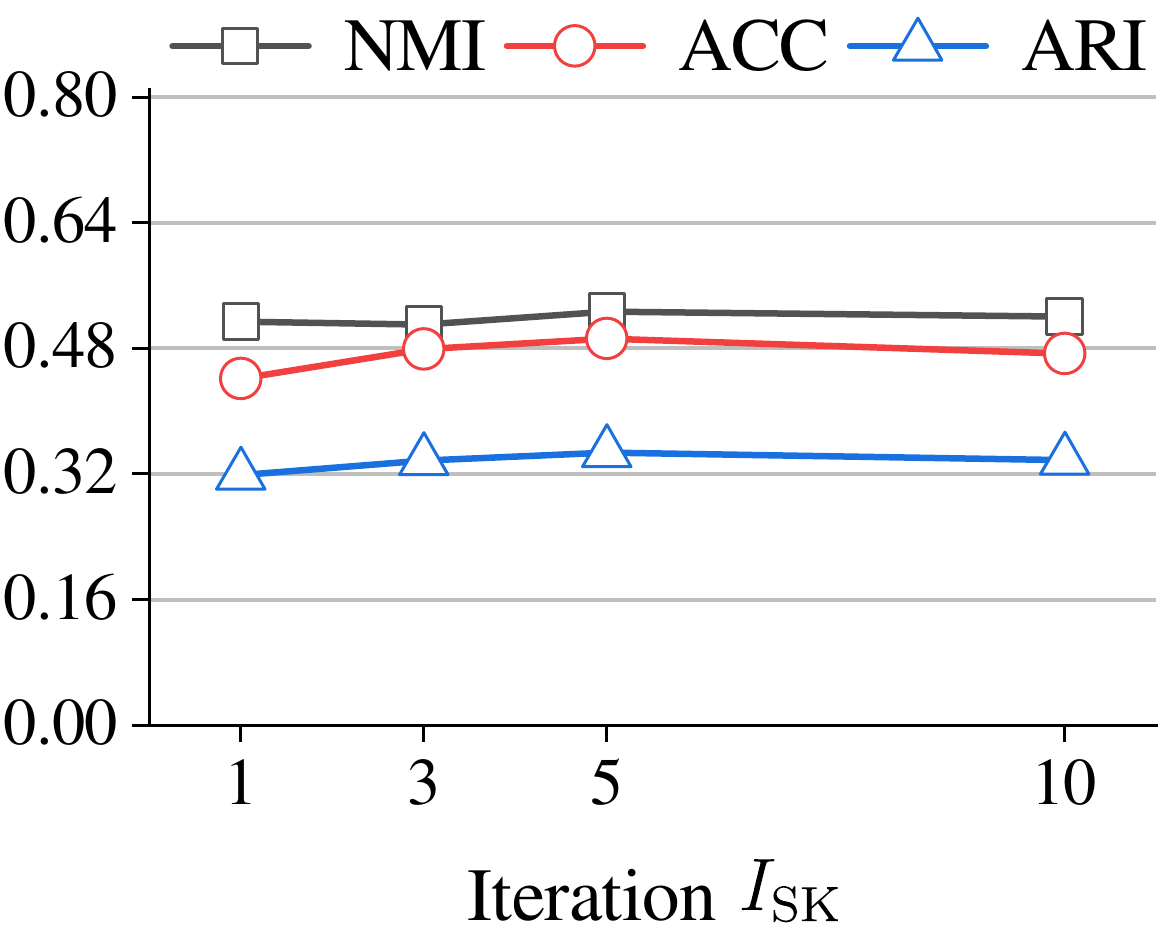}}
	\subfigure[ImageNet-10]{\includegraphics[width=\Withd\columnwidth]{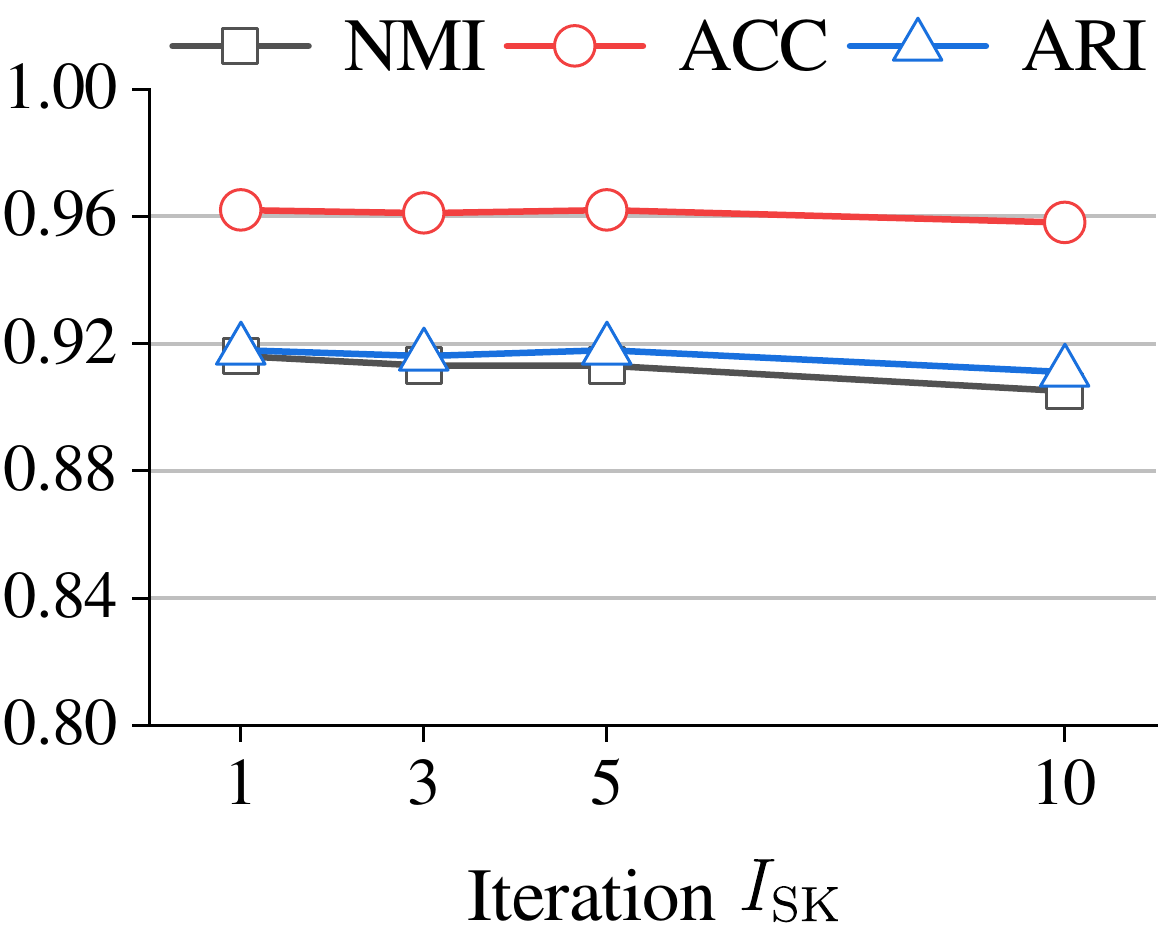}}
	\subfigure[ImageNet-Dogs]{\includegraphics[width=\Withd\columnwidth]{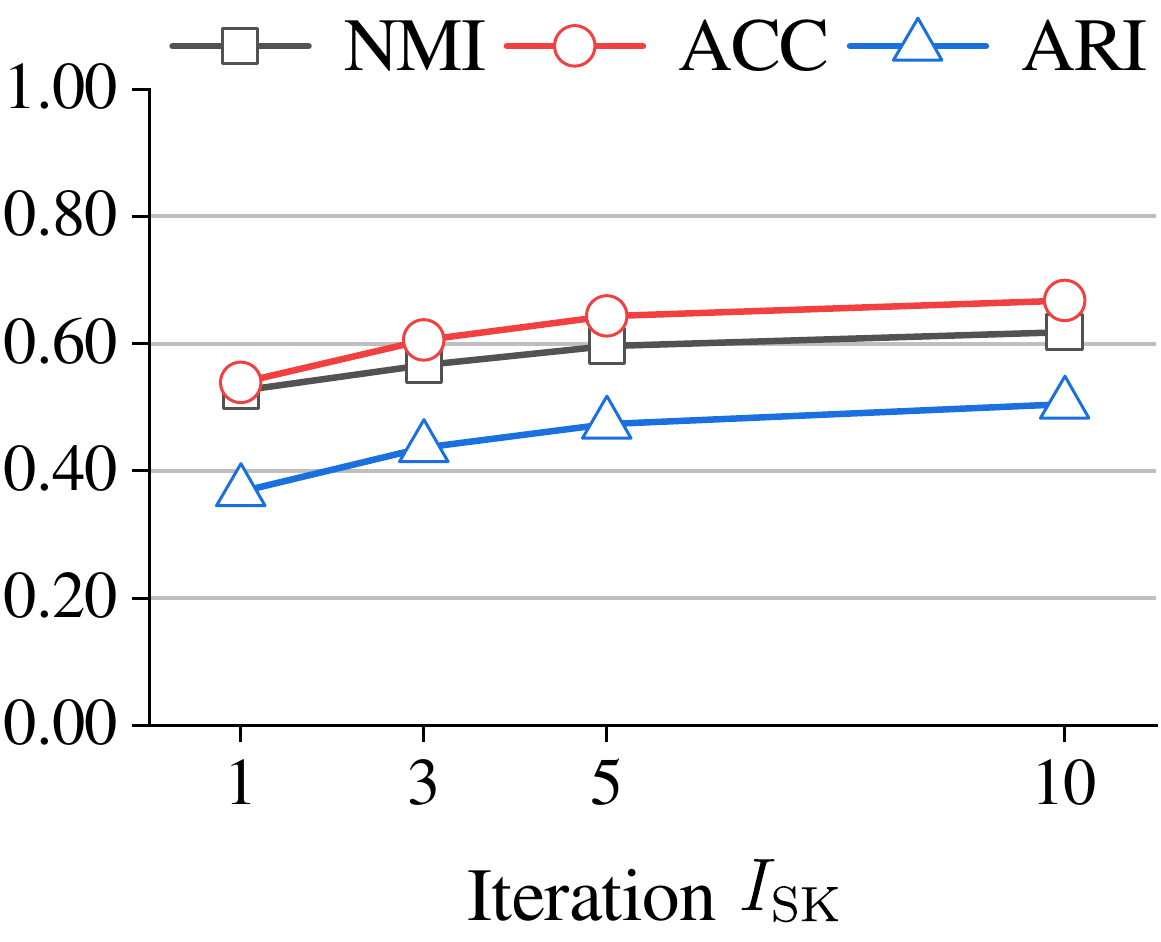}}
	\caption{Effect of the number of iterations $I_{\text{SK}}$ in Sinkhorn’s fixed point iteration.}
	\label{fig_I_sk}
\end{figure}

\noindent\textbf{Qualitative Case Study.} 
To gain a qualitative understanding of BootSC, we strictly randomly select some images from the ImageNet-10 dataset and present their predicted cluster assignments.
As depicted in Figure \ref{img_casestudy}, the top four rows show true positive cases, while the bottom two rows show false negative and false positive cases,  respectively.
Most cluster assignments of the failure cases appear reasonable.
For example, considering the false positive case of the king penguin (row 6, column 1), the shadow of the standing man exhibits certain visual patterns that frequently appear in king penguin images.
As for the false negative airliner (row 5, column 4), the image predominantly depicts the sky, lacking specific distinctive cues for any clusters.
Consequently, BootSC produces a score distribution across multiple potential clusters.
These observations suggest that BootSC learns high-level semantic information underlying the high-dimensional complex data for effective clustering.

\noindent\textbf{Embedding Orthogonality.}
We compare the proposed orthogonal re-parameterization technique with the two existing orthogonalization methods IDFO \cite{tao2020clustering} and SpectralNet \cite{shaham2018spectralnet}.
We first train BootSC without imposing orthogonality constraints on spectral embeddings, which serve as the baseline.
IDFO introduces an additional orthogonal loss term $\rho \|\mathbf{Z}^\intercal\mathbf{Z} - \mathbf{I}_D\|^2$ to the total loss function, where $\rho$ is the trade-off parameter.
We evaluate this method with different values of $\rho \in [0.5,1,2]$.
As shown in Table \ref{tab_orth}, this method fails to produce consistent performance improvements compared to the baseline.
As for SpectralNet which exploits the QR-decomposition for orthogonalization, we observe that it results in training failures in the end-to-end clustering scenario. 
These training failures are not due to the gradient instability introduced by the QR-decomposition, as we also try the straight-through estimator in this setting. 
To further investigate the underlying cause, we present how the semantic inconsistency metric $\|\mathbf{Z}-\mathbf{Z}_{\text{new}}\|_F$, the clustering loss $\mathcal{L}_{c}$, and clustering accuracy vary during training in Figure \ref{fig_inconsistency}. 
The semantic inconsistency metric of SpectralNet fails to converge and suffers from large fluctuations during batch-wise training.
This instability can be attributed to the fact that SpectralNet does not take into account the discrepancy between the original and the orthogonalized matrix, making the cluster prototypes receive significantly different embeddings across batches.
As a result, the clustering loss fails to converge, leading to degenerate clustering performance.
In contrast, our BootSC specifically minimizes the semantic inconsistency and thus achieves better clustering performance.
Interestingly, the semantic inconsistency metric of BootSC steadily decreases during training, which indicates the learned embeddings naturally tend to be orthogonal and less redundant.

\subsection{Sinkhorn's Fixed Point Iteration}
\label{sec_sk}
Solving for the optimal affinity matrix $\mathbf{W}^+$ in Equation \eqref{wq_wplus} is in essence a linear programming problem.
While the problem can be precisely solved by the network simplex algorithm \cite{bonneel2011displacement} in polynomial time, it typically results in \textit{discrete} solutions for $\mathbf{W}^+$. Discrete solutions prove suboptimal in the context of affinity learning, as they tend to pay attention to the strongest affinities while ignoring the potentially valuable insights contained in the less dominant affinities.
This is unreasonable, as each sample can indeed have multiple possible neighbors belonging to the same cluster.
Instead, we employ Sinkhorn's fixed point iteration \cite{cuturi2013sinkhorn} to solve Equation \eqref{wq_wplus}.
It yields the relaxed \textit{continuous} solutions that provide richer supervisory information by considering both the strongest and less dominant affinities, thereby facilitating model training. 
As demonstrated in Table \ref{tab_eta}, the continuous solutions consistently outperform the discrete ones.

\begin{table*}[htbp]
	\centering
	\caption{Clustering performance on the subset of the CIFAR-10 dataset with varying retention rates for each class. \textit{NA}: the result is unavailable due to training failure.}
	\begin{tabular}{c|ccc|ccc|ccc}
		\toprule
		\multirow{2}{*}{Minimum retention rate $r$} & \multicolumn{3}{c|}{0.1} & \multicolumn{3}{c|}{0.5} & \multicolumn{3}{c}{1.0} \\
		& NMI & ACC & ARI & NMI & ACC & ARI & NMI & ACC & ARI \\ \midrule
		SpectralNet  \cite{shaham2018spectralnet}& NA & NA & NA & 0.112 & 0.237 & 0.056 & 0.121 & 0.250 & 0.062 \\
		DivClust \cite{metaxas2023divclust} & 0.590&0.571&0.472& 0.640&0.680&0.540& 0.710 & 0.815 & 0.675  \\
		BootSC (Ours) & \textbf{0.694} & \textbf{0.697} & \textbf{0.612} & \textbf{0.768} & \textbf{0.843} & \textbf{0.726} & \textbf{0.814} & \textbf{0.883} & \textbf{0.780} \\ \bottomrule
	\end{tabular}
	\label{tab_imbalanced_cluster}
\end{table*}

\begin{table*}[htbp]
	\centering
	\caption{Clustering performance on the subset of the CIFAR-10 dataset with varying proportions.}
	\begin{tabular}{c|ccc|ccc|ccc}
		\toprule
		\multirow{2}{*}{Proportions $p$}  & \multicolumn{3}{c|}{0.1} & \multicolumn{3}{c|}{0.5} & \multicolumn{3}{c}{1.0} \\
		& NMI & ACC & ARI & NMI & ACC & ARI & NMI & ACC & ARI \\ \midrule
		SpectralNet \cite{shaham2018spectralnet} & 0.087 & 0.230 & 0.047 & 0.103 & 0.221 & 0.050 & {0.121} & 0.250 & 0.062 \\
		DivClust \cite{metaxas2023divclust} & 0.442 &0.493  & 0.310  &  0.597&		0.679&	0.498  & 0.710 & 0.815 & 0.675 \\
		BootSC (Ours) & \textbf{0.529} & \textbf{0.622} & \textbf{0.430} & \textbf{0.759} & \textbf{0.842} & \textbf{0.711} & \textbf{0.814} & \textbf{0.883} & \textbf{0.780} \\ \bottomrule
	\end{tabular}
	\label{tab_fewer_samples}
\end{table*}

Additionally, we evaluate the effect of the trade-off parameter $\eta$ that weights the entropic regularization term in Sinkhorn's fixed point iteration, thereby controlling the smoothness of the solutions.
We evaluate a set of $\eta$ values from 0.01 to 0.1 with a step size of  0.01.
As shown in Figure \ref{fig_eta}, the peak performance consistently occurs when $\eta$ is close to 0.05, which is a good balance between the entropic regularization and the clustering objective.
A smaller $\eta$ tends to yield more sparse solutions similar to the network simplex algorithm, which provides limited supervisory information.
In constrast, a larger $\eta$ makes the entropic regularization term dominate and yields over-smoothed solutions, which could blur the boundaries between clusters and produce underperformed results. 

Furthermore, we evaluate how the number of iterations $I_{\text{SK}} \in [1, 3, 5, 10]$ in Sinkhorn's fixed-point iteration affects the performance.
Figure \ref{fig_I_sk} indicates that \methodname is robust to the value of $I_{\text{SK}}$, which is attributed to the rapid convergence property of Sinkhorn’s fixed point iteration. 
We recommend fixing $I_{\text{SK}}=5$, at which value Sinkhorn's fixed-point iteration consistently produces near-optimal performance across various datasets. 

\subsection{Evaluation on Imbalanced and Limited Data}
\noindent\textbf{Imbalanced Sample Sizes.}
To study the effect of imbalanced sample sizes on the performance of BootSC, we sample subsets of the CIFAR-10 dataset with varying retention rates.
For a given minimum retention rate $r$, samples of the $k$-th class, $k\in [1,2,\cdots,10]$, are retained with probability $r\times {k}/{10}$.
Consequently, the smallest cluster (class $1$) is approximately $r$ times as large as the largest (class $10$).
We compare our BootSC with the popular deep spectral clustering method SpectralNet \cite{shaham2018spectralnet} and the state-of-the-art deep clustering method DivClust \cite{metaxas2023divclust}.
As shown in Table \ref{tab_imbalanced_cluster}, BootSC consistently outperforms SpectralNet and DivClust across all retention rates, which indicates the robustness of BootSC in handling datasets with unbalanced sample sizes.

\noindent\textbf{Limited Sample Sizes.}
To study the effect of the number of samples on the performance of BootSC, we sample subsets of the CIFAR-10 dataset with varying proportions.
For a given proportion $p$, each sample is retained with probability $p$. This yields a subset with approximately $p\times60,000$ samples, where $60,000$ is the total number of samples in the CIFAR-10 dataset.
As shown in Table \ref{tab_fewer_samples}, BootSC consistently outperforms SpectralNet and DivClust across all proportions, which indicates the effectiveness of BootSC in handling datasets with limited sample sizes.

	\section{Conclusion}
In this paper, we have proposed an end-to-end deep spectral clustering model BootSC that eliminates the need for pre-trained networks and works well for large datasets.
BootSC exploits a novel bootstrap procedure to jointly refine the affinity matrix and the cluster assignment matrix and leverages a semantically consistent re-parameterization technique for spectral embedding orthogonalization.
The experimental results demonstrate that BootSC consistently outperforms the state-of-the-art methods on five benchmark datasets.
The comprehensive ablation study and parameter sensitivity analysis further validate the contribution of its key components toward overall performance improvements. 
Despite these promising results, we recognize that real-world applications like medical image clustering may pose additional challenges due to potential heterogeneous data distributions, privacy constraints, and other factors.
We plan to incorporate external knowledge (stored in Large Language Models) as supervision to further boost spectral clustering.
Moreover, real-world applications may operate in resource-constrained environments like edge devices.
In the future, we would like to integrate BootSC with techniques like model quantization and knowledge distillation to improve efficiency while maintaining clustering performance.

	\section*{Acknowledgments}
		We thank the anonymous reviewers for their valuable and constructive comments. This work was supported partially by the National Natural Science Foundation of China (grant \#62176184) and the Fundamental Research Funds for the Central Universities.
	
	\bibliographystyle{IEEEtran}
	\bibliography{reference}
	
	\begin{IEEEbiography}[{\includegraphics[width=1in,height=1.25in,trim={0in 0in 0in 0in},clip,keepaspectratio]{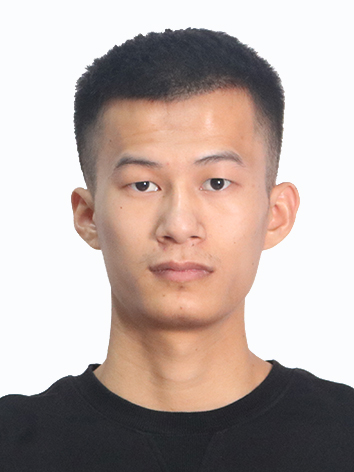}}]{Wengang Guo}
		is a PhD student in the College of Electronic and Information Engineering at Tongji University, Shanghai, China. He obtained his Bachelor's degree in Automation from Shanghai DianJi University, Shanghai, China, in 2020. His research interests primarily focus on clustering and computer vision.
	\end{IEEEbiography}	
	
	\begin{IEEEbiography}[{\includegraphics[width=1in,height=1.25in,trim={0.7in 0in 0.18in 0in},clip,keepaspectratio]{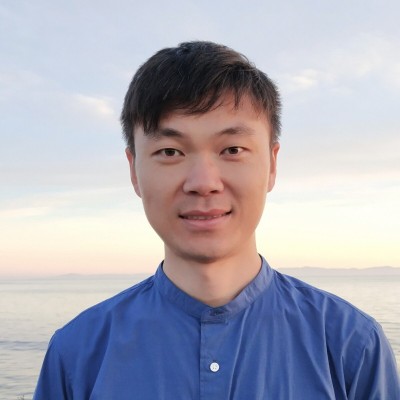}}]{Wei Ye} received the PhD degree in Computer Science from Institut f\"{u}r Informatik, Ludwig-Maximilians-Universit\"{a}t M\"{u}nchen, Munich, Germany, in 2018. He is a tenure-track professor with the College of Electronic and Information Engineering at Tongji University, Shanghai, China, Frontier Science Center for Intelligent Autonomous Systems, Ministry of Education, China, and Shanghai Innovation Institute, Shanghai, China. He was a postdoctoral researcher with the DYNAMO lab at University of California, Santa Barbara, from 2018 to 2020. Before that, he worked as a researcher in the Department of AI Platform, Tencent, China. His research interests include data mining, graph-based machine learning, deep learning, and network science.
	\end{IEEEbiography}
	
	\begin{IEEEbiography}[{\includegraphics[width=1in,height=1.25in,trim={0in 0in 0in 0in},clip,keepaspectratio]{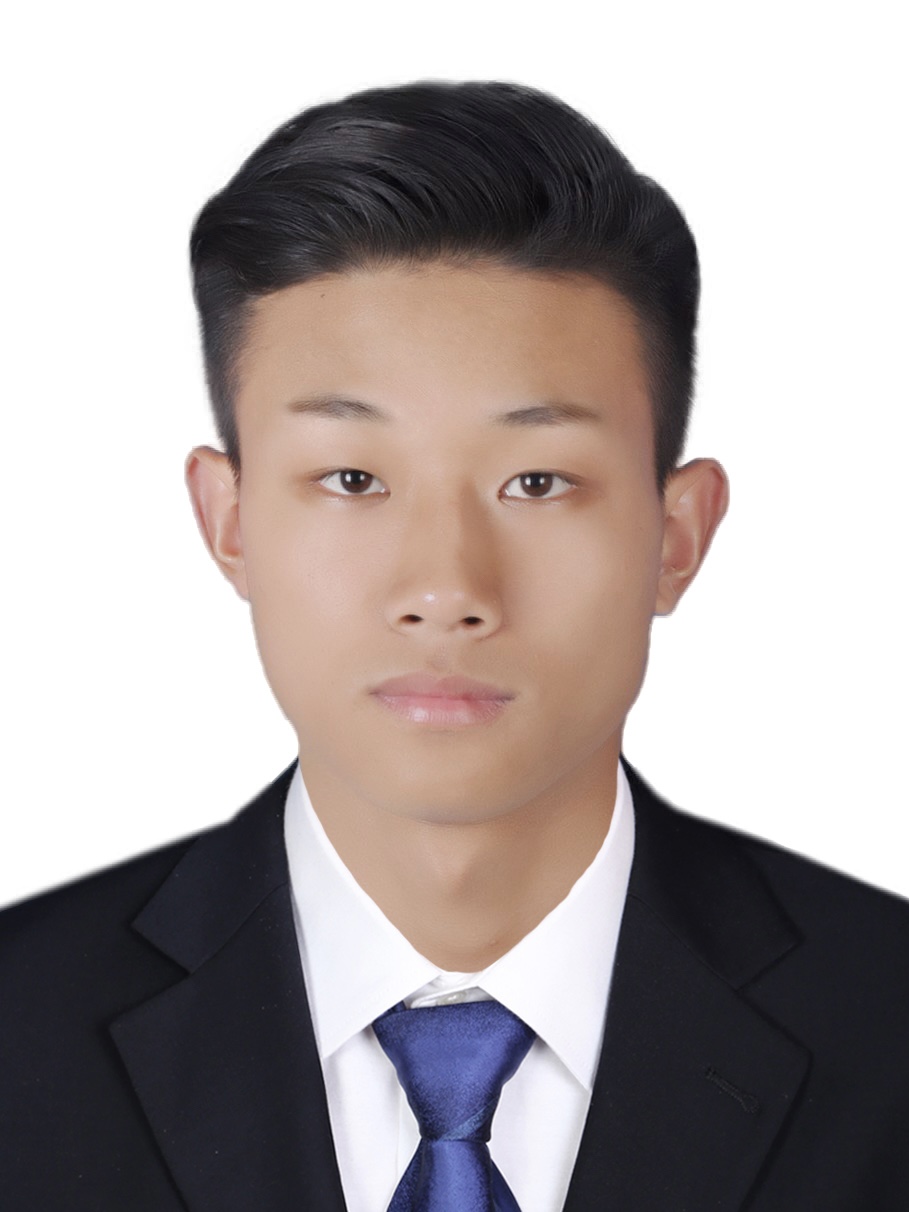}}]{Chunchun Chen} received the ME degree in computer technology from China Jiliang University, Hangzhou, China, in 2023. He is currently working toward the PhD degree with the Shanghai Research Institute for Intelligent Autonomous Systems, Tongji University, Shanghai, China. His research interests include machine learning, unsupervised graph learning, community detection.
	\end{IEEEbiography}	
	\vspace{-5mm}
    
	\begin{IEEEbiography}[{\includegraphics[width=1in,height=1.25in,trim={0in 0in 0in 0in},clip,keepaspectratio]{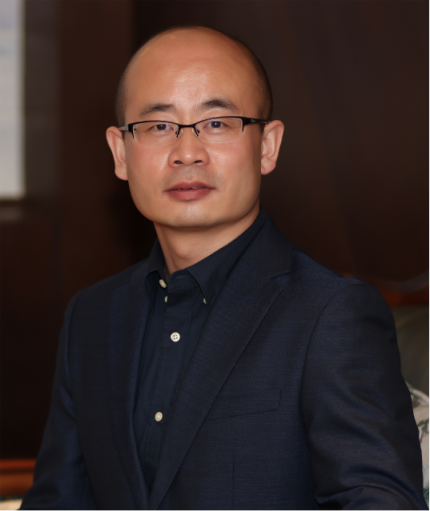}}]{Xin Sun} (Member, IEEE) received the bachelor’s, M.Sc., and Ph.D. degrees from the College of Computer Science and Technology, Jilin University, Changchun, China, in 2007, 2010, and 2013, respectively.
	He is a Full Professor with the Faculty of Data Science, City University of Macau, Macau, China. He was an experienced Humboldt Researcher at the Technical University of Munich (TUM), Munich, Germany, from 2022 to 2023. He did the Post-Doctoral Research at the Department of Computer Science, Ludwig-Maximilians-Universität München, Munich, from 2016 to 2017. His research interests include machine learning, remote sensing, and computer vision.
	\end{IEEEbiography}	
	\vspace{-5mm}
    
	\begin{IEEEbiography}[{\includegraphics[width=1in,height=1.25in,trim={0in 0in 0in 0in},clip,keepaspectratio]{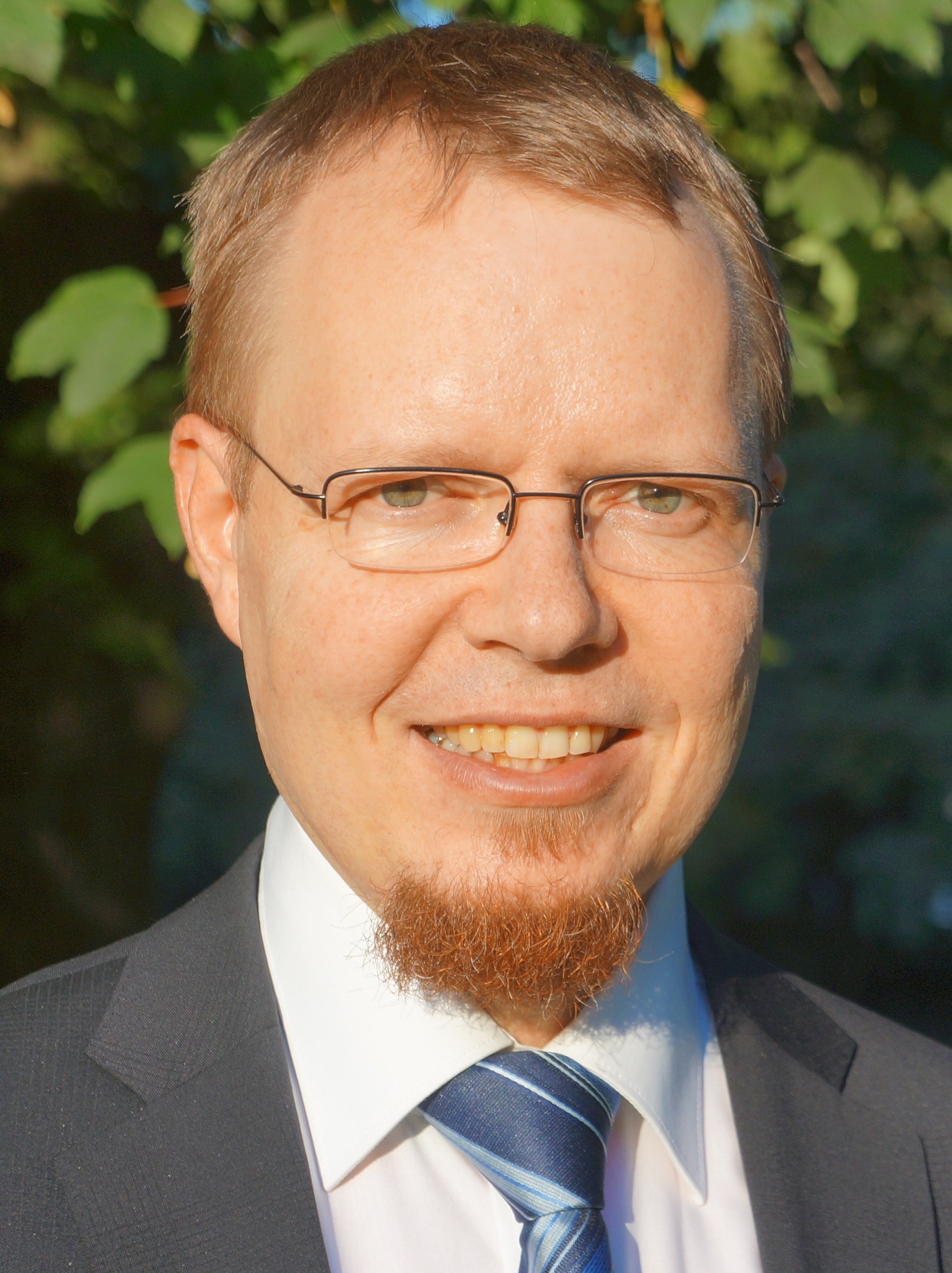}}]{Christian B\"ohm} received the PhD degree, in 1998 and the habilitation degree, in 2001. He is currently a professor of computer science	with the University of Vienna, Vienna, Austria. His research interests include database systems and data mining, particularly index structures for	similarity search and clustering algorithms. He has received several research awards in the top-tier data mining conferences.
	\end{IEEEbiography}	
	\vspace{-5mm}
    
	\begin{IEEEbiography}[{\includegraphics[width=1in,height=1.25in,trim={0in 0in 0in 0in},clip,keepaspectratio]{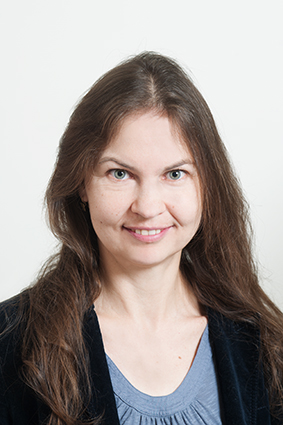}}]{Claudia Plant} received the PhD degree, in 2007. She is currently a professor of computer science with the University of Vienna, Vienna, Austria. Her research focuses on databases and data mining,	especially clustering, information-theoretic data mining, and integrative mining of heterogeneous	data. She has received several best paper awards in the top-tier data mining conferences.
	\end{IEEEbiography}	
	\vspace{-5mm}
    
	\begin{IEEEbiography}[{\includegraphics[width=1in,height=1.25in,trim={0in 0in 0in 0in},clip,keepaspectratio]{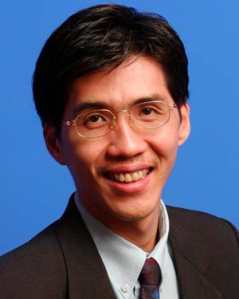}}]{Susanto Rahardja} (Fellow, IEEE) is currently a Professor of engineering cluster with the Singapore Institute of Technology, Singapore. His research interests include multimedia coding and processing, wireless communications, discrete transforms, machine learning, signal processing algorithms, and implementation and optimization. He contributed to the development of a series of audio compression technologies, such as Audio Video Standards AVS-L, AVS-2, ISO/IEC 14496-3:2005/Amd.2:2006, and ISO/IEC 14496-3:2005/Amd.3:2006, which have licensed worldwide. Mr. Rahardja has more than 15 years of experience in leading a research team in the above-mentioned areas. He was an Associate Editor for IEEE TRANSACTIONS ON AUDIO, SPEECH AND LANGUAGE PROCESSING, and Senior Editor for IEEE JOURNAL OF SELECTED TOPICS IN SIGNAL PROCESSING. He is an Associate Editor for Elsevier Journal of Visual Communication and Image Representation, IEEE TRANSACTIONS ON INDUSTRIAL ELECTRONICS, IEEE TRANSACTIONS ON MULTIMEDIA, and Member of Editorial Board of IEEE ACCESS. He is a Fellow of the Academy Engineering, Singapore.
	\end{IEEEbiography}	
	
\end{document}